\documentclass[letterpaper]{article}
\usepackage{aaai24} 
\usepackage{times} 
\usepackage{helvet} 
\usepackage{courier} 
\usepackage[hyphens]{url} 
\usepackage{graphicx} 
\usepackage{graphicx} 
\usepackage{natbib} 
\usepackage{caption} 
\usepackage{enumitem}
\usepackage[svgnames]{xcolor}

\newcommand{\paragraphHack}[1]{\paragraph{#1}}

\newcommand{\bld}[1]{\boldmath\textbf{#1}\unboldmath}

\usepackage{xr} 
\usepackage[export]{adjustbox}
\usepackage{lscape}
\usepackage{makecell}
\usepackage{xspace}
\usepackage{cuted}
\usepackage{bm}

\usepackage[ruled,vlined,linesnumbered]{algorithm2e} 
\SetKwProg{function}{Function}{}{}

\usepackage{amsfonts}
\usepackage{amsmath}
\usepackage{amssymb}
\usepackage{amsthm}
\usepackage{mathrsfs}

\usepackage{tikz-network}
\usepackage{tikz}
\usetikzlibrary{fit, calc}
\usetikzlibrary{shapes.multipart}  

\newcommand{\astar}{A$^*$\xspace}
\newcommand{\ilao}{iLAO$^*$\xspace}
\newcommand{\cgilao}{CG-\ilao}
\newcommand{\lrtdp}{LRTDP\xspace}
\newcommand{\ftvi}{FTVI\xspace}
\newcommand{\ilaoH}[1]{iLAO\(^*_{\text{#1}}\)\xspace}
\newcommand{\cgilaoH}[1]{CG-\ilaoH{#1}}
\newcommand{\lrtdpH}[1]{LRTDP\(_{\text{#1}}\)\xspace}

\newcommand{\bestCovr}[1]{\ensuremath{\bm{#1}}}
\newcommand{\bestMean}[1]{\bm{#1}}

\newcommand{\hOptPerturbedW}{\ensuremath{h^{\text{pert}}_w}\xspace}

\newcommand{\supp}{\ensuremath{\text{succ}}}
\newcommand{\definedas}{\ensuremath{\coloneqq}}

\newcommand{\partssptuple}{\ensuremath{\langle \partstates, \sZ, \partgoals, \partactions, \pr,\C, \h \rangle}\xspace}
\newcommand{\partssp}{\ensuremath{\widehat{\ssp}}\xspace}
\newcommand{\partstates}{\ensuremath{\widehat{\Ss}}\xspace}
\newcommand{\partgoals}{\ensuremath{\widehat{\Sg}}\xspace}
\newcommand{\partactions}{\ensuremath{\widehat{\A}}\xspace}

\newcommand{\bw}{BW\xspace}
\newcommand{\exbw}{ExBW\xspace}
\newcommand{\tireworld}{\ensuremath{\text{TW}}\xspace}
\newcommand{\parc}{PARC\xspace}
\newcommand{\parcn}{PARC-N\xspace}
\newcommand{\parcr}{PARC-R\xspace}

\newcommand{\duals}{\ensuremath{\bm{\mathcal{V}}}\xspace}

\newcommand{\currp}{\ensuremath{\widehat{\p}_{\V}}\xspace}
\newcommand{\oldp}{\ensuremath{\widehat{\p}_{\text{old}}}\xspace}
\newcommand{\colsToCheck}{\ensuremath{\Gamma}\xspace}
\newcommand{\envelope}{\ensuremath{\mathcal{E}}\xspace}

\newcommand{\residual}{\ensuremath{\textsc{res}}\xspace}

\newcommand{\improvePolicy}{\ensuremath{\textsc{Backups}}}
\newcommand{\CGimprovePolicy}{\ensuremath{\textsc{CG-Backups}}}
\newcommand{\addStateActions}{\ensuremath{\textsc{Add-Actions}}}
\newcommand{\expandFringes}{\textsc{Expand-Fringes}}
\newcommand{\partiallyExpandFringes}{\textsc{Partly-Expand-Fringes}}
\newcommand{\fixViolatedConstrs}{\textsc{Fix-Constrs}}
\newcommand{\predecessors}{\ensuremath{\textsc{Preds}}\xspace}
\newcommand{\successors}{\ensuremath{\textsc{Ext-Succs}}\xspace}

\newcommand{\V}{\ensuremath{V}\xspace}
\newcommand{\Q}{\ensuremath{Q}\xspace}
\newcommand{\Qsa}{\ensuremath{Q(\s, \ac)}\xspace}
\newcommand{\qvalue}{\Q-value\xspace}
\newcommand{\qvalues}{\Q-values\xspace}

\newcommand{\econsistent}{\ensuremath{\epsilon\text{-consistent}}\xspace}
\newcommand{\econsistency}{\ensuremath{\epsilon\text{-consistency}}\xspace}

\newtheorem{definition}{Definition}

\newtheorem{theorem}{Theorem}

\newtheorem{lemma}{Lemma}

\usepackage{cleveref}
\crefname{algorithm}{alg.}{algs.}
\crefname{line}{line}{lines}
\crefname{definition}{defn.}{defns.}
\crefname{lemma}{lem.}{lem.s}
\crefname{theorem}{thm.}{thm.s}
\crefname{figure}{fig.}{figs.}
\crefname{table}{tab.}{tabs.}

\def\argmin{\mathop{\rm argmin}}

\newcommand{\setSymbol}[1]{\ensuremath{\mathsf{#1}}}

\newcommand{\probSymbol}[1]{\ensuremath{\mathbb{#1}}}

\newcommand{\Reals}{\ensuremath{\mathbb{R}}\xspace}

\newcommand{\Rp}{\ensuremath{\Reals_{> 0}}\xspace}

\newcommand{\s}{\ensuremath{s}\xspace}
\newcommand{\sZ}{\ensuremath{s_0}\xspace}

\newcommand{\p}{\ensuremath{\pi}\xspace}

\newcommand{\ssp}[1][]{\ifthenelse{\equal{#1}{}}{\probSymbol{S}\xspace}{\ensuremath{\probSymbol{S}^{#1}}\xspace}}
\newcommand{\Ss}{\setSymbol{S}\xspace}
\newcommand{\Sg}{\St}
\newcommand{\St}{\ensuremath{\setSymbol{G}}\xspace}
\newcommand{\A}{\setSymbol{A}\xspace}
\newcommand{\ac}{\ensuremath{a}\xspace}
\newcommand{\pr}{\ensuremath{P}\xspace}
\newcommand{\C}[1][]{\ifthenelse{\equal{#1}{}}{\ensuremath{C}}{\ensuremath{C_{#1}}}\xspace}
\newcommand{\sspTuple}{\ensuremath{\langle \Ss, \sZ, \St, \A, \pr,\C\rangle}\xspace}

\newcommand{\hMax}{\ensuremath{h^\text{max}}\xspace}
\mathchardef\mhyphen="2D
\newcommand{\hLMcut}{\ensuremath{h^\text{lmc}}\xspace}

\newcommand{\h}{\ensuremath{H}\xspace}
\newcommand{\fringe}{\setSymbol{F}\xspace}

\usepackage{mathtools}

\newcommand{\hROC}{\ensuremath{h^\text{roc}}\xspace}

\makeatletter 

\newcommand{\LPresized}[2]{\centerline{\resizebox{\linewidth}{!}{\begin{LPenv}{#1}#2\end{LPenv}}}}

\newcounter{constraintcounter}
\newcommand{\@tagconstrain}[1]{\refstepcounter{constraintcounter}\tag{C\theconstraintcounter}\label{#1}}

\newcommand{\@constr}[2]{& \mathrlap{#1} & #2 \notag}
\newcommand{\@@constr}[3]{& \mathrlap{#1} & #2 \@tagconstrain{#3}}
\newcommand{\@longconstr}[2]{& \mathrlap{#1} \notag \\[-3mm] & & #2 \notag}
\newcommand{\@@longconstr}[3]{& \mathrlap{#1} \notag \\[-3mm] & & #2 \@tagconstrain{#3}}

\newcounter{lpcounter}
\newcommand{\@taglp}[1]{\refstepcounter{lpcounter}\tag{LP~\thelpcounter}\label{#1}}
\newcommand{\@obj}[2]{#1 \quad & \mathrlap{#2} \notag}
\newcommand{\@@obj}[3]{#1 \quad & \mathrlap{#2} \@taglp{#3}}

\newcommand{\@st}{\mathrm{s.t.} \quad}

\newenvironment{LPenv}[1]
{\minipage[c]{#1}%
\def\st{\@st}%
\def\obj{\@ifstar\@obj\@@obj}%
\def\constr{\@ifstar\@constr\@@constr}%
\def\longconstr{\@ifstar\@longconstr\@@longconstr}%
\flalign}
{\endflalign\endminipage}

\makeatother

\usepackage{tabularx}
\usepackage{booktabs}

\title{Efficient Constraint Generation for Stochastic Shortest Path Problems}
\author {
       Johannes Schmalz,
    Felipe Trevizan
}
\affiliations {
       School of Computing, Australian National University \\
    johannes.schmalz@anu.edu.au, felipe.trevizan@anu.edu.au
}

\begin{document}

\maketitle

\begin{abstract}
Current methods for solving Stochastic Shortest Path Problems (SSPs) find states' costs-to-go by applying Bellman backups, where state-of-the-art methods employ heuristics to select states to back up and prune.
A fundamental limitation of these algorithms is their need to compute the cost-to-go for every applicable action during each state backup, leading to unnecessary computation for actions identified as sub-optimal.
We present new connections between planning and operations research and, using this framework, we address this issue of unnecessary computation by introducing an efficient version of constraint generation for SSPs.
This technique allows algorithms to ignore sub-optimal actions and avoid computing their costs-to-go.
We also apply our novel technique to \ilao resulting in a new algorithm, \cgilao.
Our experiments show that \cgilao ignores up to \(57\%\) of \ilao's actions and it solves problems up to \(8\times\) and \(3\times\) faster than \lrtdp and \ilao.

\end{abstract}

\section{Introduction}

Planning is an important facet of AI that gives efficient algorithms for solving current real-world problems.
Stochastic Shortest Path problems (SSPs)~\cite{Bertsekas1991:SSPs} generalise classical (deterministic) planning by introducing actions with probabilistic effects, which lets us model problems where the actions are intrinsically probabilistic.
Value Iteration (VI)~\cite{Bellman57} is a dynamic programming algorithm that forms the basis of optimal algorithms for solving SSPs.
VI finds the cost-to-go for each state, which describes the solution of an SSP.
A state \s's cost-to-go is the minimum expected cost of reaching a goal from \s, and similarly a action \ac's cost-to-go is the minimum after applying \ac.
VI finds the optimal cost-to-go by iteratively applying \emph{Bellman backups}, which update each state's cost-to-go with the minimal outgoing action's cost-to-go.

\lrtdp~\cite{Bonet2003:lrtdp} and \ilao~\cite{Hansen2001:ilao}, the state-of-the-art algorithms for optimally solving SSPs, build on VI and offer significant speedup by using heuristics to apply Bellman backups only to promising states and pruning states that are deemed too expensive.
A shortcoming of such algorithms is that each Bellman backup must consider all applicable actions.
For instance, let \s and \ac be a state and an applicable action, even if all successors of \ac will be pruned because they are too expensive, a Bellman backup on \s still computes the \qvalue of \ac, so these algorithms can prune unpromising states but not actions.
This issue is compounded because algorithms for SSPs require arbitrarily many Bellman backups on each state \s to find the optimal solution, thus wasting time on computing \qvalues for such actions many times.

This issue of computing unnecessary \qvalues for a state \s is addressed by \emph{action elimination}~\cite{Bertsekas1995}, which can be implemented in search algorithms to prune useless actions.
Action elimination looks for pairs~\((\ac, \widehat{\ac})\) of applicable actions in a state, such that a lower bound on \(\widehat{\ac}\)'s cost-to-go exceeds an upper bound on \ac's cost-to-go, in which case \(\widehat{\ac}\) is proved to be a useless action and can be pruned.
Although domain-independent lower bounds~(heuristics) can be computed efficiently, finding an efficient, domain-independent upper bound remains an open question to the best of our knowledge.
This gap has limited the use of action elimination in domain-independent planning.
In the context of optimal heuristic planning for SSPs, the only algorithm we are aware of that utilises action elimination to prune actions is \ftvi~\cite{Dai2009:ftvi}.
Other algorithms, such as BRTDP~\cite{McMahan2005:BRTDP}, FRTDP~\cite{smith06:frtdp} and VPI-RTDP~\cite{sanner09:vpirtdp}, use upper bounds in conjunction with lower bounds to guide their search.
However, unlike FTVI, they do not perform action elimination.

We present a general technique for ignoring actions that does not rely on upper bounds.
In contrast to action elimination that incrementally prunes useless actions, our approach initially treats all actions as inactive, i.e., not contributing to the solution.
It then iteratively adds actions to the search that are guaranteed to improve the current solution.
To develop our approach, we strengthen the connections between planning and operations research by relating heuristic search to \emph{variable} and \emph{constraint generation}.
Similar to heuristic search, variable and constraint generation enable the solving of large Linear Programs (LPs) by considering only a subset of variables and constraints.
We show that algorithms such as \ilao implicitly perform constraint generation, albeit in a trivial manner, to the LP encoding of VI.
Building on this, we introduce an efficient algorithm for constraint generation for SSPs that leads to inactive actions being ignored.
We apply our approach to \ilao to get the novel algorithm: \cgilao.

In our experiments, \cgilao solves problems up to \(8\times\) and \(3\times\) faster than \lrtdp and \ilao, respectively.
 \cgilao is faster than the others over various problem difficulties: its improvement is apparent from problems that require only 4 minutes to solve, and the improvement gap increases as problems take longer to solve.
To explain this, we quantify that \cgilao only considers 43--65\% of \ilao's total actions thus fewer actions' costs-to-go are computed.
Investigating further, we empirically show that \cgilao combines \ilao's efficient use of backups and \lrtdp's strong pruning capability, thereby displaying the best characteristics of both in a single algorithm.

The structure of this paper is as follows: we first introduce the background for SSPs and existing methods for solving SSPs.
Second, we give a brief background to linear programming and connect linear programming techniques to heuristic search.
Then we explain and motivate our novel algorithm \cgilao and prove its correctness.
Finally, we empirically evaluate the performance of \cgilao.

\section{Background}

A Stochastic Shortest Path problem (SSP)~\cite{Bertsekas1991:SSPs} is a tuple \sspTuple where:
\Ss is a finite set of states;
\(\sZ \in \Ss\) is the initial state;
\(\Sg \subset \Ss\) are the goal states with \(\Sg \neq \emptyset\);
\(\A\) is a finite set of actions and \(\A(\s) \subseteq \A\) denotes the actions applicable in state \s;
\(\pr(\s'|\s, \ac)\) gives the probability that applying action \ac in state \s results in state \(\s'\);
\(\C(\s, \ac) \in \Rp\) is the cost of applying \ac in \s.

The states immediately reachable by applying \ac to \s are called successors and are given by \(\supp(\s, \ac) \definedas \{\s' \in \Ss : \pr(\s' | \s, \ac) > 0\}\).
A solution to an SSP is given by a map \(\p: \Ss \to \A\), called a policy.
A policy \p is \textit{closed} w.r.t. \sZ if each state \s that can be reached by following \p from \sZ is either a goal or \p is defined for \s.
A policy is \textit{proper} w.r.t. \sZ if it reaches the goal with probability 1 from \sZ and it is \textit{improper} otherwise.
An optimal policy \(\p^*\) is any proper policy that minimises the expected cost to reach the goal from the initial state \sZ.

For simplicity, we make two standard assumptions: (i) there exists at least one proper policy w.r.t. \sZ, this is called the reachability assumption; and (ii) all improper policies have infinite expected cost.
A consequence of assumption (ii) is that $\A(\s) \neq \emptyset$ for all $\Ss \setminus \Sg$.
Note that we define \C to be strictly positive in order to avoid zero-cost cycles that would violate assumption (ii).
In our experiments, we relax the reachability assumption by applying the fixed-penalty transformation of SSPs~\cite{trevizan17:mcmp} resulting in a new SSP without dead ends.
Other approaches for handling SSPs such as S3P~\cite{Teichteil-Koenigsbuch2012:S3P} are also compatible with our approach.

An SSP's optimal solution is uniquely represented by the optimal value function \(\V^*\), which is the unique fixed-point solution to the Bellman equations:
\begin{equation}\label{eqn:bellman-equation}
\V(\s) =
\begin{cases}
0                            &\text{ if } \s \in \Sg \\
\min_{\ac \in \A(\s)} Q(s,a) &\text{otherwise}
\end{cases}
\quad
\forall \s \in \Ss
\end{equation}
where \(\Qsa \definedas \C(\s, \ac) + \sum_{\s' \in \A(\ac)} \pr(\s'|\s, \ac) \V(\s)\) is known as the \qvalue of \s and \ac.
A (optimal) value function \(\V(s)\) and the associated \qvalue \Qsa represent the (minimum) expected cost to reach the goal from state \s and after executing action \ac on state \s, respectively.
Given a value function \V, the policy associated with it is defined as \(\p_{\V}(\s) \definedas \argmin_{\ac \in \A(\s)} \Qsa\) and is known as the greedy policy for \V.
Ties can be broken arbitrarily thanks to assumption (ii), and for simplicity we assume some tie-breaking rules that ensure the greedy policy is unique.

The Bellman equations~\eqref{eqn:bellman-equation} can be iteratively solved by Value Iteration (VI)~\cite{Bellman57}:
VI starts with an arbitrary value function $\V^0$ and computes $\V^{t+1}(s) \definedas \min_{\ac \in \A(\s)} \Qsa$ over all states, where \Qsa uses the previous value function $\V^t$.
This process of computing $\V^{t+1}(s)$ using $V^t$ is called a Bellman backup.
VI is guaranteed to asymptotically converge to $\V^*$ regardless of $\V^0$.
For practical reasons, VI is terminated at iteration $t$ when the \emph{Bellman residual} $\residual(\s) \definedas |\V^t(\s) - \min_{a \in \A(s)} \Qsa|$ is less than or equal to $\epsilon \in \Rp$ for all $\s \in \Ss$.

Given a policy \p, its policy envelope $\Ss^\p \subseteq \Ss$ is the set of all reachable states when following \p from \sZ.
Note that VI explores the complete state space \Ss regardless of the optimal policy envelope \(\Ss^{\p*}\)'s size.
To address this shortcoming, heuristic search algorithms such as \ilao~\cite{Hansen2001:ilao} and \lrtdp~\cite{Bonet2003:lrtdp} use a heuristic function \h to initialise \(\V^0\), which guides the exploration of the state space \Ss in a way that expands as few states as possible.
To find \(\V^*\), heuristic search algorithms require the heuristic to be \emph{admissible}, that is, \(\h(s) \le \V^*(s) \; \forall \s \in \Ss\).
 Often heuristics are also \emph{monotonic}, which immediately implies admissibility.
A value function \V is monotonic if \(\V(\s) \leq \min_{\ac \in \A(\s)} \Qsa \; \forall \s \in \Ss\), and the definition is analogous for \h.
Similar to VI, heuristic search algorithms converge to \(\V^*\) asymptotically and require a practical stop criterion.
This stop criterion is known as \econsistency~\cite{Bonet2003:lrtdp} and is defined as:
\begin{definition}[\econsistency]\label{def:epsilon-consistency}
A value function \V is \econsistent if \(\residual(s) \le \epsilon \; \forall \s \in \Ss^{\p_{\V}}\).
\end{definition}

\noindent Notice that \econsistency checks the residual only on the policy envelope of a greedy policy and states outside the envelope are permitted to have a residual larger than $\epsilon$.

\begin{algorithm}[!t]
{\small
\DontPrintSemicolon
  \caption{\ilao}\label{alg:ilao}
          \function{\ilao\((\ssp, \h, \epsilon)\)} {
       \(\partssp \gets\) partial SSP \(\langle \{\sZ\}, \sZ, \{\sZ\}, \emptyset, \pr, \C, \h \rangle\) \;
    \(\V \gets \text{Value Function initialised by } \h\) \;
       \Repeat{\(\fringe = \emptyset\) and \(\oldp = \currp\) and \(\residual < \epsilon\)} {
      \(\envelope \gets\) post-order DFS traversal of \currp from \sZ \; \label{line:ilao:dfs}
      \(\partssp \gets \expandFringes{}(\ssp, \partssp, \envelope)\) \; \label{line:ilao:call-to-expandFringes}
      \(\fringe \gets \Ss^{\currp} \cap (\partgoals \setminus \Sg)\) \;
      \(\V, \residual, \oldp \gets \improvePolicy{}(\partssp, \envelope, \V, \fringe)\) \; \label{line:ilao:call-to-improvePolicy}
    }
    \Return \V
  }

     \function{\(\expandFringes{}(\ssp, \partssp, \envelope)\)}{
    \For{\(\s_f \in \envelope \cap (\partgoals \setminus \Sg)\)} {
           \(\partgoals \gets \partgoals \setminus \{\s_f\}\) \;
      \(\addStateActions(\ssp, \partssp, \s_f, \A(\s_f))\) \; \label{line:expandFringes:addStateActions}
                             } \label{line:expandFringes:column-generation-begin}
    \Return \partssp \label{line:ilao:expandFringes-end}
    \label{line:expandFringes:column-generation-end}
  } \label{line:ilao:expandFringes-begin}

  \function{\(\addStateActions{}(\ssp, \partssp, \s, \A')\)}{
    \(\partactions(\s) \gets \partactions(\s) \cup \A'\) \; \label{line:expandFringes:constraint-generation-end}
    \For{\(\ac \in \A'\)} {
      \(\partgoals \gets \partgoals \cup (\supp(\s, \ac) \setminus \partstates)\) \;
      \(\partstates \gets \partstates \cup \supp(\s, \ac)\) \;
    }
  }

     \function{\(\improvePolicy{}(\partssp, \envelope, \V, \fringe)\)} {
    \(\oldp \gets \currp\) \;
    \Repeat{\(\residual < \epsilon\) or \(\currp \neq \oldp\) or \(\fringe \neq \emptyset\)} {
      \(\residual \gets 0\) \;
      \For {\(\s \in \envelope \setminus \partgoals\)} {
        \(\Q_{\min} \gets \min_{\ac \in \partactions(\s)} \Qsa\) \;
        \(\residual \gets \max(|V(\s) - \Q_{\min}|, \residual)\) \;
        \(\V(\s) \gets \Q_{\min}\) \;
      }
    }
    \Return \(\V, \residual, \oldp\)
  }
}
\end{algorithm}

We close this section by reviewing \ilao~\cite{Hansen2001:ilao}.
\ilao~(\cref{alg:ilao}) is an iterative algorithm which works by incrementally growing a \emph{partial SSP}:\footnote{Called the explicit graph in the original paper.}
\begin{definition}[Partial SSP]\label{def:partial-ssp}
Given an SSP \sspTuple and a heuristic \h, a partial SSP \partssp is an SSP with \emph{terminal costs} defined by \(\partssp = \partssptuple\) with \(\partstates \subseteq \Ss, \partgoals \subset \partstates, \Sg \cap \partstates \subseteq \partgoals, \partactions(\s) \subseteq \A(\s) \; \forall \s \in \partstates\) and terminal cost \h.
\end{definition}

SSPs with terminal costs have a one-time terminal cost of \(\h(\widehat{g})\) that is incurred when \(\widehat{g} \in \partgoals\) is reached, so their Bellman equations are \eqref{eqn:bellman-equation} with the goal case replaced by \(\V(\widehat{g}) = \h(\widehat{g})\) for all \(\widehat{g} \in \partgoals\).
It is trivial to see that SSPs with terminal costs are equivalent to SSPs, and we use them to simplify our presentation.
In a partial SSP \partssp, we refer to states in $\partgoals \setminus \St$ as artificial goals, and we define \currp to be the greedy policy over \V restricted to \(\partstates \setminus \partgoals\).

At each iteration, \ilao expands its partial SSP \partssp by expanding the artificial goals reachable by \currp into regular states.
To expand \(\widehat{g} \in \partgoals \setminus \Sg\), \ilao adds \(\widehat{g}\)'s applicable actions to \partssp and adds new reachable states as artificial goals (\cref{alg:ilao} \cref{line:ilao:call-to-expandFringes}).
This artificial goal expansion is done to make \currp eventually closed w.r.t. \sZ for the original SSP \ssp.
Simultaneously, \ilao also works towards making \V \econsistent by applying a Bellman backup to all the states reachable by \currp (\cref{alg:ilao} \cref{line:ilao:call-to-improvePolicy}).
These Bellman backups are ordered by a post-order DFS traversal of \currp, so states that occur closer to artificial goals are updated first, and \sZ is updated last.
Note that, when a state is expanded, it may have a successor \s already within the partial SSP.
If this happens, the DFS in \cref{alg:ilao}~\cref{line:ilao:dfs} must keep traversing \currp from \s to ensure the policy envelope \envelope is accurate.

In the next section, we show how to interpret \ilao through the lens of Operations Research by relating it to techniques used for handling large linear programs.

\section{\ilao as Linear Program}

SSPs can be solved by the Linear Program (LP) presented in~\ref{lp:vi}.
This LP is known as the primal LP or VI LP since it directly encodes the Bellman equations.

\LPresized{9.8cm}{
\obj{\max_{\duals}}{V_{\sZ}}{lp:vi} \\
\st
\constr{\V_{\s} \leq \C(\s,\ac) +\!\sum_{\mathclap{\s' \in \Ss}}\pr(\s'|\s,\ac)\V_{\s'}}{\forall \s \in \Ss\!\setminus\!\St, \ac\!\in\!\A(\s)}{c:vi:consistency}\\
\constr{\V_g \leq 0}{\forall g \in \St}{c:vi:goal}
}\vspace{3mm}

\noindent
Each variable \(\V_{\s} \in \duals\) represents \(\V(\s)\), and for each state \(\s \in \Ss \setminus \Sg\) the relevant constraints \ref{c:vi:consistency} encode \(\V(\s) \leq \min_{\ac \in \A(\s)} \Qsa\).
When clear from context, we use \(\V(\s)\) to represent \(\V_{\s}\), and \Qsa for the right-hand side of \s and \ac's constraint~\ref{c:vi:consistency}.
Together with the objective that maximises \(\V(\sZ)\), we obtain the Bellman equations~\eqref{eqn:bellman-equation} for the states in the optimal policy envelope \(\Ss^{\p*}\), i.e., the constraints are active (tight) for the pairs \((\s, \p^*(\s))\) for all \(\s \in \Ss^{\p*}\) and inactive (slack) everywhere else.
We reframe \ilao's incremental growing of its partial SSP as
solving \ref{lp:vi} with \emph{variable and constraint generation}~\cite{Bertsimas1997:introduction}.\footnote{Variable generation is also known as column generation and constraint generation is also known as the cutting plane method.}

Variable generation is a technique from Operations Research that enables us to solve LPs with a large number of variables by considering only a subset of variables.
Given an LP with missing variables called the Reduced Master Problem (RMP), variable generation finds a set of variables outside the RMP whose addition lets the RMP's solution quality improve.
Variable generation provides a sound and complete method to select such variables and a stop criterion that ensures the optimal solution for the RMP is also optimal for the original LP.
Heuristic search algorithms, such as \ilao, can be seen as a variable generation algorithm over \ref{lp:vi}, where each of its partial SSPs represents an RMP with the subset of variables~\(\{\V_{\s} : \s \in \partstates\}\).
For \ilao, the variable selection mechanism is inherited from \astar and is represented by the expansion of the artificial goals reachable by \currp (\cref{alg:ilao} \crefrange{line:expandFringes:column-generation-begin}{line:expandFringes:column-generation-end}): for all \(s_f \in \Ss^{\currp} \cap (\partgoals \setminus \Sg)\) and \(\ac \in \A(s_f)\), we add the variables \(V_{\s}\) such that \(\s \in \supp(s_f, \ac)\) and \(\s \not\in \partstates\).

Constraint generation is a similar technique which enables the solving of LPs with a large (potentially infinite) number of constraints.
The key idea is that the optimal solution of an LP with many constraints only makes a small number of constraints active, thus only a subset of constraints is needed to characterise this optimal solution.
In constraint generation, the intermediate LPs are known as \emph{relaxed LPs} since they relax the original LP by removing one or more constraints.
Given a relaxed LP, constraint generation finds one or more constraints in the original LP that are violated by the optimal solution of the relaxed LP.
By iteratively adding these violated constraints and re-optimising the new relaxed LP, a sequence of relaxed LPs with an increasing number of constraints is generated.
When no violations are found, the optimal solution of the relaxed LP is an optimal solution for the original LP.
The algorithm used to check constraint violations is called a separation oracle and the effectiveness of constraint generation relies on the availability of an efficient separation oracle for the original LP.

In the case of \ilao, constraint generation adds all the actions of the expanded artificial goal (\cref{alg:ilao}~\cref{line:expandFringes:addStateActions}) where each action $\ac \in \A(\s_f)$ added to the partial SSP implicitly represents the constraint $V(s_f) \le Q(s_f,a)$.
This separation oracle is computationally cheap since no checks are performed to detect if this new constraint is needed or not, at the cost that inactive constraints are unnecessarily added to the partial SSP.
In the next section, we present our new algorithm that uses an efficient separation oracle that only adds violated constraints.

\section{\cgilao}

\begin{algorithm}[!t]
{\small
\DontPrintSemicolon
  \caption{\cgilao}\label{alg:cg-ilao}
          \function{\cgilao\((\ssp, \h, \epsilon)\)} {

       \(\partssp \gets\) partial SSP \(\langle \{\sZ\}, \sZ, \{\sZ\}, \emptyset, \pr, \C, \h \rangle\) \;
    \(\V \gets \text{value function initialised by } \h\) \;
                                                                               \Repeat{\(\fringe = \emptyset\) and \(\oldp = \currp\) and \(\residual \leq \epsilon\)} {
      \(\envelope \gets\) post-order DFS traversal of \currp from \sZ \; \label{line:cgilao:dfs}
                                    \(\partssp \gets \partiallyExpandFringes{}(\ssp, \partssp, \envelope, \V)\) \; \label{line:cgilao:call-to-partiallyExpandFringes}
      \(\fringe \gets \Ss^{\currp} \cap (\partgoals \setminus \Sg)\) \;
           \(\V, \residual, \oldp, \colsToCheck \gets \CGimprovePolicy{}(\ssp, \partssp, \envelope, \V, \fringe, \colsToCheck)\) \;
                \(\V, \residual, \colsToCheck, \partssp \gets \fixViolatedConstrs{}(\ssp, \partssp, V, \colsToCheck, \residual)\) \;
         } \label{line:cgilao:whileCondition}
    \Return \V
  }

  \function{\(\partiallyExpandFringes{}(\ssp, \partssp, \envelope, \V)\)}{
    \For{\(\s_f \in \envelope \cap (\partgoals \setminus \Sg)\)} {
      \(\partgoals \gets \partgoals \setminus \{\s_f\}\) \;
      $Q_{min} \gets \min_{a \in \A(\s_f)} Q(s_f,a)$ \;
      $\A' \gets \{a \in \A(s_f) | Q(s_f,a) = Q_{min}\}$ \;
      $\addStateActions(\ssp, \partssp, \s_f, \A')$ \;
    }
    \Return \partssp \;
  }

  \function{\(\CGimprovePolicy{}(\ssp, \partssp, \envelope, \V, \fringe, \colsToCheck)\)} {
    \(\oldp \gets \currp\) \;
    \Repeat{\(\residual \leq \epsilon\) or \(\currp \neq \oldp\) or \(\fringe \neq \emptyset\)}{
      \(\residual \gets 0\) \;
      \For {\(\s \in \envelope \setminus \partgoals\)} {
        \(\Q_{\min} \gets \min_{\ac \in \partactions(\s)} \Qsa\) \;
               \If {\(\Q_{\min} - \V(\s) > \epsilon\)} {
                                                                \(\colsToCheck \gets \colsToCheck \cup \successors(\s, \ssp, \partssp)\) \; \label{line:cgilao:v-increases}
                 } \label{line:cgilao:improvePolicy:check-violated-constraints-begin}
        \ElseIf {\(V(\s) - \Q_{\min} > \epsilon\)} {
                                                                                           \(\colsToCheck \gets \colsToCheck \cup \predecessors(\s, \ssp)\) \; \label{line:cgilao:v-decreases} \label{line:cgilao:improvePolicy:check-violated-constraints-end}
                 }
        \(\residual \gets \max(|\V(\s) - \Q_{\min}|, \residual)\) \;
        \(\V(\s) \gets \Q_{\min}\) \;
             }
    } \label{line:cgilao:improvePolicy:repeatCondition}
    \Return \(\V, \residual, \oldp, \colsToCheck\)
  }

  \function{\(\fixViolatedConstrs{}(\ssp, \partssp, \V, \colsToCheck, \residual)\)} {
             \(\colsToCheck' \gets \emptyset\) \;
    \For{\((\s,\ac) \in \colsToCheck \text{ s.t. } \V(\s) >  \Q(\s,\ac) + \epsilon\)} {
                                                                                 \If{\(\ac \not\in \partactions(\s)\)} {
        \(\partactions(\s) \gets \partactions(\s) \cup \{\ac\}\) \;
        \(\partgoals \gets \partgoals \cup (\supp(\s, \ac) \setminus \partstates)\) \;
        \(\partstates \gets \partstates \cup \supp(\s, \ac)\) \;
      }
      \(\residual \gets \max(V(\s) - \Qsa, \residual)\) \; \label{line:cgilao:fixViolatedConstrs:track-residual}
      \(\V(\s) \leftarrow \Qsa\) \; \label{line:cgilao:fixViolatedConstrs:fix-violation}
           \(\colsToCheck' \gets \colsToCheck' \cup \predecessors(\s, \ssp)\) \; \label{line:cgilao:fixViolatedConstrs:check-violated-constraints}
         }
    \Return \(\V, \residual, \colsToCheck', \partssp\)
  }
}
\end{algorithm}

In \ilao, as most search algorithms, each state is either unexpanded or fully expanded, i.e., either none or all of its applicable actions are considered and it is not possible to ignore just a subset of the applicable actions.
We start by defining which actions can be safely ignored in a partial SSP.

\begin{definition}[Inactive Action]
Consider an SSP \ssp, its partial SSP \partssp, a value function \V, and a state $\s \in \partstates$.
An action $\ac \in \A(\s) \setminus \partactions(\s)$ is inactive in state \s if $V(s) < Q(s,a).$
\end{definition}

An inactive action $\ac \in \A(\s) \setminus \partactions(\s)$ for $s \in \partstates$ represents the inactive constraint \ref{c:vi:consistency} for $s$ and $a$ in \ref{lp:vi}.
Since inactive actions are not in the partial SSP and their associated constraints are inactive, adding them to the partial SSP does not change the solution and only adds overhead in the form of \qvalue computation for sub-optimal actions.

We generalise \ilao by allowing states to be partially expanded, so these states only have a subset of actions available in the partial SSP.
Under the lens of linear programming, we use constraint generation to identify and add actions that may be needed to encode the optimal solution and to ignore inactive actions.
We call this algorithm Constraint-Generation \ilao (\cgilao).

\cgilao is presented in \cref{alg:cg-ilao}.
One of the defining changes from \ilao is in \cgilao's expanding phase (\cref{alg:cg-ilao} \cref{line:cgilao:call-to-partiallyExpandFringes}) where \partiallyExpandFringes{} only expands a state with the greedy actions on \V, rather than all the applicable actions.
This introduces two challenges:
(i) actions that were not added by the partial expansion may need to be added later when \V is more accurate; and
(ii) when we add such actions, \V must be updated to reflect \ac's availability.
If (ii) is not addressed, the reduction to \V offered by \ac is not propagated, potentially leading to a suboptimal solution since $V$ would overestimate \(\V^*\).
The key insight of our algorithm is that both of these challenges are instances of constraint violation.
Thus, we can solve both issues by finding which constraints are violated with a separation oracle and enforcing them in the style of constraint generation.

The trivial separation oracle checks all constraints \((\s, \ac)\) for \(\s\!\in\!\partstates\) and \(\ac\!\in\!\A(\s)\) for violations; this is needlessly expensive since some non-violated constraints remain non-violated from one iteration to the next.
Our separation oracle exploits this persistence between iterations by tracking changes in \V to compute a subset of constraints which could potentially be violated by the following rules (\cref{alg:cg-ilao} \crefrange{line:cgilao:improvePolicy:check-violated-constraints-begin}{line:cgilao:improvePolicy:check-violated-constraints-end} and \cref{line:cgilao:fixViolatedConstrs:check-violated-constraints}):
suppose \(\V(\s)\) is assigned \(\min_{\ac \in \partactions(\s)} \Qsa\), then there are three cases:

\begin{enumerate}
\item{\bld{\(\V(\s)\) stays the same.}}
No new constraint violations.

\item{\bld{\(V(\s)\) increases.}}
The constraints \(\V(\s) \leq \Q(\s, \ac')\) may be violated for \((\s, \ac') \in \successors(\s, \ssp, \partssp) \definedas \{(\s, \ac) : \ac \in \A(\s) \setminus \partactions(\s)\}\).
Note that if \((\s, \ac')\) is already inside \partssp, i.e., \(\ac' \in \partactions(\s)\), its constraint can not be violated since \(\V(\s) \gets \min_{\ac \in \partactions(\s)} \Qsa \leq \Q(\s, \ac')\).

\item{\bld{\(V(\s)\) decreases.}}
The only constraints that may be violated are \(\V(\s') \leq \Q(\s', \ac')\) for \((\s', \ac') \in \predecessors(\s, \ssp) \definedas \{(\s', \ac') : \s \in \supp(\s', \ac'), \ac' \in \A(\s')\}\).
\end{enumerate}

We store potential violations in \colsToCheck, i.e., if \(\V(\s)\) increases, the elements of \(\successors(\s, \ssp, \partssp)\) are added to \colsToCheck (\cref{alg:cg-ilao}~\cref{line:cgilao:v-increases}); and if \(\V(\s)\) decreases, the elements of \(\predecessors(\s, \ssp)\) are added to \colsToCheck (\cref{alg:cg-ilao}~\cref{line:cgilao:v-decreases}).
Elements are removed from \colsToCheck after they are checked.
As we prove later in this section, checking constraints in \colsToCheck is sufficient to find any constraint violations.

\cgilao fixes a violated constraint \((\s, \ac)\) by setting \(\V(\s) \gets \Qsa\) (\cref{alg:cg-ilao} \cref{line:cgilao:fixViolatedConstrs:fix-violation}).
This change in \V may create a new violation in another state, so we must track such potential violations in the same way as before.
This ensures that all constraint violations are tracked and fixed eventually before termination.

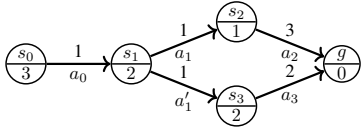
\begin{figure}[t]
\centering
\scalebox{0.7}{
\begin{tikzpicture}
\begin{scope}[every node/.style={circle split, thick, draw, minimum size=0.6cm, inner sep=0.04cm}]
        \node (0) at (0, 0) {\(\sZ\) \nodepart{lower} \(3\)};
  \node (1) at (2, 0) {\(\s_1\) \nodepart{lower} \(2\)};
        \node (2) at (4, 0.8) {\(\s_2\) \nodepart{lower} \(1\)};
        \node (3) at (4, -0.8) {\(\s_3\) \nodepart{lower} \(2\)};
        \node (g) at (6, 0) {\(g\) \nodepart{lower} \(0\)};
\end{scope}
\begin{scope}[every edge/.style={draw=black, very thick}]
        \path [->] (0) edge[] node[above] {1} node[below] {\(\ac_0\)} (1);
        \path [->] (1) edge[] node[above] {1} node[below] {\(\ac_1\)} (2);
        \path [->] (1) edge[] node[above] {1} node[below] {\(\ac'_1\)} (3);
        \path [->] (2) edge[] node[above] {3} node[below] {\(\ac_2\)} (g);
        \path [->] (3) edge[] node[above] {2} node[below] {\(\ac_3\)} (g);
\end{scope}
\end{tikzpicture}}
\caption{An SSP where \cgilao's value function is not monotonically non-decreasing.}
\label{fig:v-decrease-example}
\end{figure}

Note that $V(s)$ may decrease after an update (case 3 of the constraint violations) in \cgilao even if the heuristic used is monotonic.
This is a departure from all other algorithms based on Bellman backups where \V is guaranteed to be monotonically non-decreasing during their executions when initialised with a monotonic heuristic.
To illustrate a scenario where \V decreases in \cgilao, consider the SSP in \cref{fig:v-decrease-example} where \h is monotonic and represented inside nodes.
The first iterations of \cgilao applied to this SSP are:

\begin{enumerate}[label=Iter. \arabic*, leftmargin=*]

\item expands \sZ.
\item partly expands \(\s_1\) with \(\ac_1\).
\item expands \(\s_2\) with \(\ac_2\) and, after \CGimprovePolicy{}, we have \(\V(\s_2) = 3\), \(\V(\s_1) = 4\), \(\V(\s_0) = 5\) and \(\colsToCheck = \{(\s_1, \ac'_1)\}\). Since \(\colsToCheck \neq \emptyset\), \fixViolatedConstrs{} verifies that \((\s_1, \ac'_1)\) is currently better than the existing action \(\ac_1\) for \(\s_1\), so \(\ac'_1\) is added to \(\partactions(\s_1)\) and \(\V(\s_1)\) is changed from 4 to 3.
Recall that \(\V(\s_0) = 5\), so when \(\V(\s_1)\) is updated to \(3\), \(\{\s_0, \ac_0\}\) is inserted into \colsToCheck, and no further changes are made in this iteration.
\item expands \(\s_3\) and \CGimprovePolicy{} reduces \(\V(\sZ)\) from 5 to 4, so \V has been decreased by \CGimprovePolicy{}.

\end{enumerate}

\cgilao generalises \ilao by using a more precise separation oracle that only adds violated constraints, which translates to \cgilao ignoring inactive actions.
For a state \s, any action that has been ignored and left out of the partial SSP \partssp is not considered by a Bellman backup of \s in \partssp, so such actions' \qvalues are not computed.
However, \cgilao needs to compute additional \qvalues in its separation oracle to check violations.
As our experiments in \cref{sec:experiments} show, the \qvalues saved by ignoring inactive actions outweigh the additional \qvalues in the separation oracle, which lets \cgilao outperform \ilao.

To close the section, we prove that \cgilao: (i) terminates (\cref{thm:cgilao-terminates}); (ii) tracks all constraint violations (\cref{lem:violations-are-in-gamma}); and (iii) returns an \econsistent value function (\cref{thm:cgilao-econsistent}).
Thus, \cgilao is optimal for SSPs.

\begin{theorem}\label{thm:cgilao-terminates}
\cgilao terminates.
\end{theorem}

\begin{proof}
For contradiction, suppose \cgilao does not terminate.
\ssp is finite and we do not add duplicate states nor constraints, so eventually \partssp is fixed
and \(\fringe = \emptyset\).
Then there must be a finite set of states \(X \subseteq \partstates\) that are updated with Bellman
backups infinitely often by \CGimprovePolicy{} and/or \fixViolatedConstrs{}.
But \(X\) induces a new partial SSP, and applying Bellman backups infinitely often to all \(X\) solves this new partial SSP with VI, so \V must converge to a fixed point and the residual will be less than \(\epsilon\) in finite time.
Thus, \V will not be updated, and all remaining termination conditions will be satisfied, giving us the desired contradiction.
\end{proof}

\begin{lemma}\label{lem:violations-are-in-gamma}
If there is \(\s \in \partstates \setminus \partgoals\) and \(\ac \in \A(\s)\) such that \(\V(\s) > \Qsa + \epsilon\), then \((\s, \ac) \in \colsToCheck\).
\end{lemma}

\begin{proof}
We prove by induction over \(n\), the number of updates to \V.
In the base case, \(n = 0\), \partssp is the initial partial SSP with \(\partstates \setminus \partgoals =
\emptyset\), so the claim is vacuously true.
Now, we show the claim holds after \(n+1\) updates to \V, assuming that the claim holds
for \(n\) updates.
Any violations must have been introduced in the latest update to \V by \CGimprovePolicy{} or
\fixViolatedConstrs{}, but we add any potential violations to \colsToCheck in \cref{alg:cg-ilao}
\crefrange{line:cgilao:improvePolicy:check-violated-constraints-begin}{line:cgilao:improvePolicy:check-violated-constraints-end}
and \cref{line:cgilao:fixViolatedConstrs:check-violated-constraints} respectively.
\end{proof}

\begin{theorem}\label{thm:cgilao-econsistent}
\cgilao outputs an \econsistent \V.
\end{theorem}

\begin{proof}
For contradiction, suppose \cgilao has terminated and outputs \V with \(\s \in \Ss^{\currp}\) such that \(\residual(\s) > \epsilon\).
By \cgilao's termination condition (\cref{alg:cg-ilao} \cref{line:cgilao:whileCondition}) we know that \(\currp = \oldp\) and \(\fringe = \emptyset\), so \CGimprovePolicy{} applies Bellman backups to all states in the envelope until \(\residual \leq \epsilon\) (\cref{alg:cg-ilao} \cref{line:cgilao:improvePolicy:repeatCondition}).
Therefore, the inconsistency of \s must be introduced by \fixViolatedConstrs{}, either directly by updating \V, or indirectly by forcing a policy change.
But the residual is tracked (\cref{alg:cg-ilao} \cref{line:cgilao:fixViolatedConstrs:track-residual}) and policy changes are flagged when \(\currp \neq \oldp\), which are both checked in the termination condition.
So, \fixViolatedConstrs{} can not introduce any inconsistency either.
But these two methods are the only ones affecting \V, which yields the desired contradiction.
This proves \econsistency (\cref{def:epsilon-consistency}), but previous heuristic search methods rely on the invariant \(\V \leq \V^*\) to safely prune states that can not be part of an optimal policy's envelope, which \cgilao does not have.
We must ensure that states \s outside the policy envelope with \(\V(\s) > \V^*(\s)\) can not lead to a cheaper policy if we apply more Bellman backups to them.
Consider such \s outside the greedy policy envelope with \(\V(\s) > \V^*(\s)\), and for contradiction let \(\V(\s) > \Qsa + \epsilon\) for some \(\ac \in \A(\s)\).
Since states in \partgoals are initialised with an admissible \h, we know that \(\s \in \partstates \setminus \partgoals\), so \((\s, \ac) \in \colsToCheck\) by \cref{lem:violations-are-in-gamma}.
Since \fixViolatedConstrs{} overwrites \colsToCheck, \((\s, \ac)\) must have been added in the previous call, but then \(\residual \gets \max(\V(\s) - \Qsa, \residual) > \epsilon\) (\cref{alg:cg-ilao} \cref{line:cgilao:fixViolatedConstrs:track-residual}), so the termination criteria (\cref{alg:cg-ilao} \cref{line:cgilao:whileCondition}) are not satisfied, giving us a contradiction.
Therefore, all inadmissible states \s satisfy \(V(\s) \leq \Qsa + \epsilon\).
Thus, if \cgilao terminates with \(\residual \leq \epsilon\), additional backups to states with \(\V(\s) > \V^*(\s)\) would not change \V, so we can conclude that \cgilao outputs \econsistent \V.
\end{proof}

\vspace{1mm}
\section{Experiments}\label{sec:experiments}

In this section we empirically compare \cgilao to two state-of-the-art optimal heuristic search planners: \ilao~\cite{Hansen2001:ilao} and \lrtdp~\cite{Bonet2003:lrtdp}.
We also compared \cgilao against \ftvi~\cite{Dai2009:ftvi}, the only algorithm we are aware of that uses action elimination to prune actions, but it is uncompetitive and FTVI's results are reported in \cref{sec:appendix}.
We consider the following admissible heuristics: h-max (\hMax); lm-cut (\hLMcut)~\cite{Helmert2009:lmcut}; and h-roc (\hROC)~\cite{trevizan17:hpom}.
As in~\cite{trevizan17:hpom}, we use \hROC with \hMax as a dead-end detection mechanism for problems with dead ends.
We use $\epsilon = 0.0001$ and convert SSPs into dead-end free SSPs~\cite{trevizan17:mcmp} with a penalty of $D = 500$ for all domains except \emph{Parc Printer variants} where $D = 10^7$ due to the large cost of single actions.

On all problems, we collected 50 runs with different random seeds for each combination of planner and heuristic.
We refer to a problem paired with a fixed seed as an instance.
All runs have a cutoff of 30 minutes of CPU time and 8GB of memory.
The experiments were conducted in a cluster of Intel Xeon 3.2 GHz CPUs and each run used a single CPU core.
The LP solver used for computing \hROC was CPLEX version 20.1.
We consider the following domains:

\newcommand\triangleScale{0.6}
\newcommand\labelTightness{0.1cm}
\begin{figure}
\centering
\scalebox{\triangleScale}{
\begin{tikzpicture}
\Vertex[x=0.350,y=0.350,color={230,230,230},label=1-1,shape=rectangle,RGB]{l-1-1}
\Vertex[x=4.150,y=0.350,color={230,230,230},label=1-3,shape=rectangle,RGB]{l-1-3}
\Vertex[x=2.250,y=0.350,color={230,230,230},label=1-2,shape=rectangle,RGB]{l-1-2}
\Vertex[x=1.300,y=1.500,color={230,230,230},label=2-1,shape=circle,RGB]{l-2-1}
\Vertex[x=3.200,y=1.500,color={230,230,230},label=2-2,shape=circle,RGB]{l-2-2}
\Vertex[x=2.250,y=2.650,color={230,230,230},label=3-1,shape=circle,RGB]{l-3-1}
\node [left = \labelTightness of l-1-1]{A};
\node [right = \labelTightness of l-3-1]{B};
\node [right = \labelTightness of l-1-3]{C};
\Edge[,Direct](l-1-1)(l-1-2)
\Edge[,Direct](l-1-1)(l-2-1)
\Edge[,Direct](l-1-2)(l-1-3)
\Edge[,Direct](l-1-2)(l-2-2)
\Edge[,Direct](l-2-1)(l-1-2)
\Edge[,Direct](l-2-1)(l-3-1)
\Edge[,Direct](l-2-2)(l-1-3)
\Edge[,Direct](l-3-1)(l-2-2)
\end{tikzpicture}
\begin{tikzpicture}
\Vertex[x=0.350,y=0.350,color={230,230,230},label=1-1,shape=rectangle,RGB]{l-1-1}
\Vertex[x=7.150,y=0.350,color={230,230,230},label=1-5,shape=rectangle,RGB]{l-1-5}
\Vertex[x=2.050,y=0.350,color={230,230,230},label=1-2,shape=rectangle,RGB]{l-1-2}
\Vertex[x=3.750,y=0.350,color={230,230,230},label=1-3,shape=rectangle,RGB]{l-1-3}
\Vertex[x=5.450,y=0.350,color={230,230,230},label=1-4,shape=rectangle,RGB]{l-1-4}
\Vertex[x=1.200,y=1.425,color={230,230,230},label=2-1,shape=circle,RGB]{l-2-1}
\Vertex[x=2.900,y=1.425,color={230,230,230},label=2-2,shape=circle,RGB]{l-2-2}
\Vertex[x=4.600,y=1.425,color={230,230,230},label=2-3,shape=circle,RGB]{l-2-3}
\Vertex[x=6.300,y=1.425,color={230,230,230},label=2-4,shape=circle,RGB]{l-2-4}
\Vertex[x=2.050,y=2.500,color={230,230,230},label=3-1,shape=circle,RGB]{l-3-1}
\Vertex[x=5.450,y=2.500,color={230,230,230},label=3-3,shape=circle,RGB]{l-3-3}
\Vertex[x=3.750,y=2.500,color={230,230,230},label=3-2,shape=rectangle,RGB]{l-3-2}
\Vertex[x=2.900,y=3.575,color={230,230,230},label=4-1,shape=circle,RGB]{l-4-1}
\Vertex[x=4.600,y=3.575,color={230,230,230},label=4-2,shape=circle,RGB]{l-4-2}
\Vertex[x=3.750,y=4.650,color={230,230,230},label=5-1,shape=circle,RGB]{l-5-1}
\node [left = \labelTightness of l-1-1]{A};
\node [right = \labelTightness of l-5-1]{B};
\node [right = \labelTightness of l-1-5]{C};
\Edge[,Direct](l-1-1)(l-1-2)
\Edge[,Direct](l-1-1)(l-2-1)
\Edge[,Direct](l-1-2)(l-1-3)
\Edge[,Direct](l-1-2)(l-2-2)
\Edge[,Direct](l-1-3)(l-1-4)
\Edge[,Direct](l-1-3)(l-2-3)
\Edge[,Direct](l-1-4)(l-1-5)
\Edge[,Direct](l-1-4)(l-2-4)
\Edge[,Direct](l-2-1)(l-1-2)
\Edge[,Direct](l-2-1)(l-3-1)
\Edge[,Direct](l-2-2)(l-1-3)
\Edge[,Direct](l-2-3)(l-1-4)
\Edge[,Direct](l-2-3)(l-3-3)
\Edge[,Direct](l-2-4)(l-1-5)
\Edge[,Direct](l-3-1)(l-3-2)
\Edge[,Direct](l-3-1)(l-2-2)
\Edge[,Direct](l-3-1)(l-4-1)
\Edge[,Direct](l-3-3)(l-2-4)
\Edge[,Direct](l-3-2)(l-3-3)
\Edge[,Direct](l-3-2)(l-4-2)
\Edge[,Direct](l-4-1)(l-3-2)
\Edge[,Direct](l-4-1)(l-5-1)
\Edge[,Direct](l-4-2)(l-3-3)
\Edge[,Direct](l-5-1)(l-4-2)
\end{tikzpicture}
}
\caption{Triangle Tire World problems 1 (left) and 2 (right).}
\label{fig:tireworld}
\end{figure}
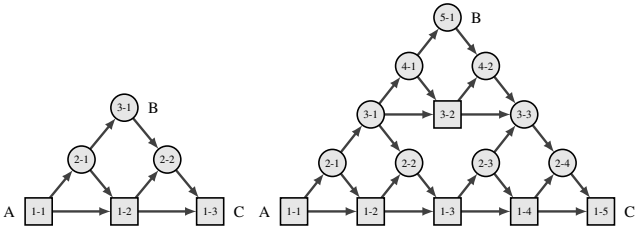

\begin{figure*}[ht!]
\includegraphics[scale=0.502, valign=t]{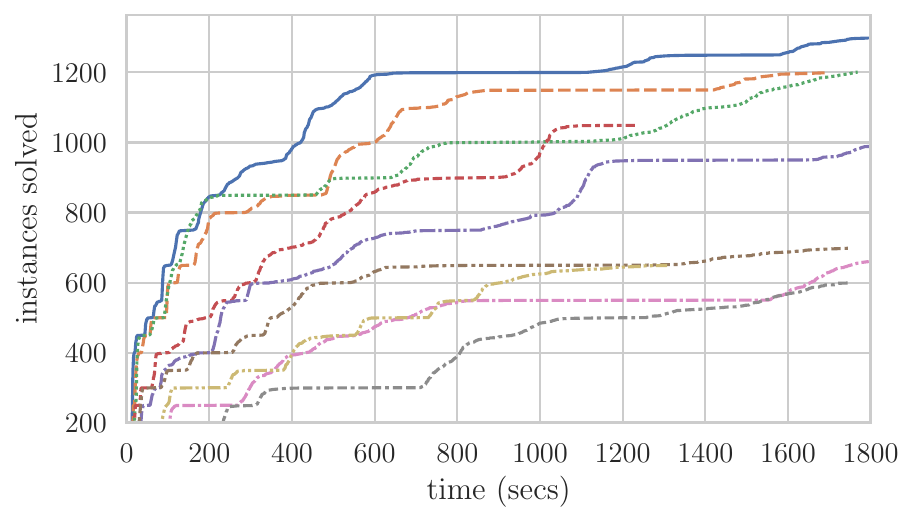}
\includegraphics[scale=0.502, valign=t]{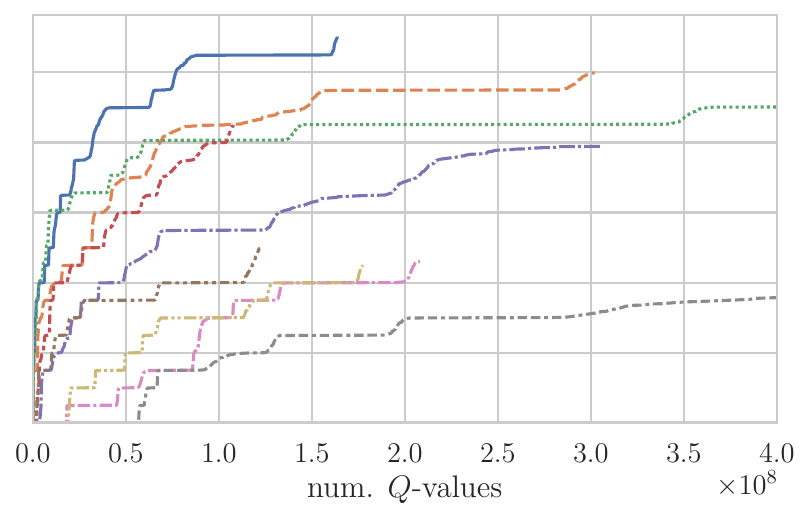}
\includegraphics[scale=0.502, valign=t]{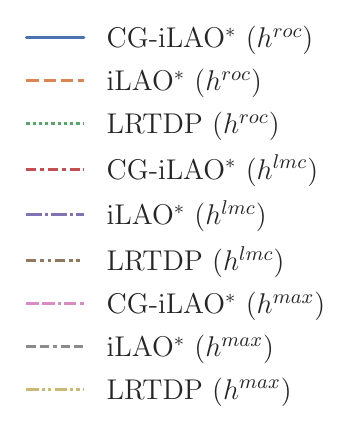}
\caption{For each algorithm and heuristic, the cumulative plot of how many instances were solved w.r.t. time in seconds (left) and number of \qvalues (right).
Both plots start at 200 solved instances and (right) is cut off at \(4 \times 10^8\) \qvalues.}
\label{fig:cumulative-planner-and-heuristic}
\end{figure*}

\begin{figure*}[ht!]
\includegraphics[scale=0.42]{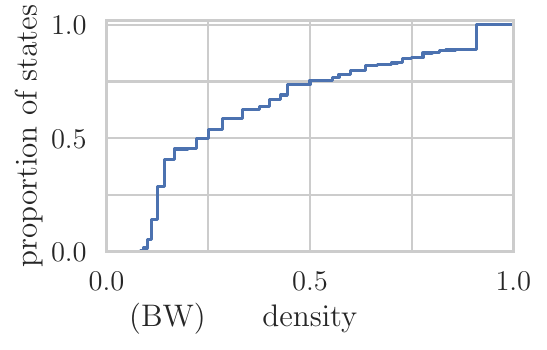}
\includegraphics[scale=0.42]{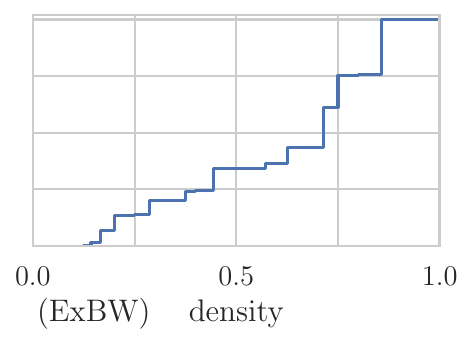}
\includegraphics[scale=0.42]{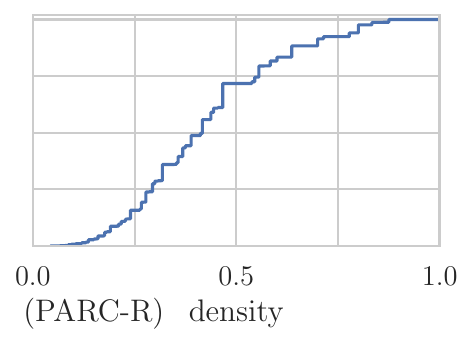}
\includegraphics[scale=0.42]{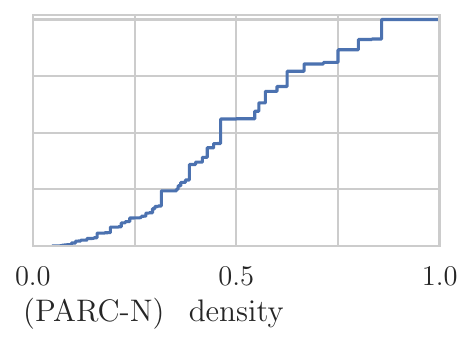}
\includegraphics[scale=0.42]{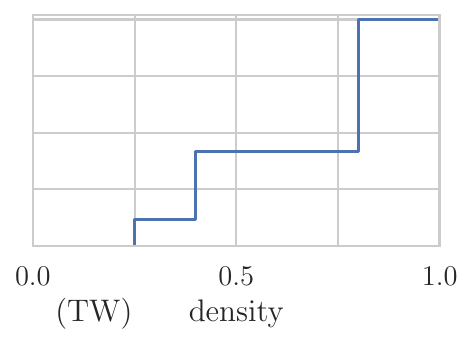}
\caption{Cumulative plot of state density, i.e., \({|\partactions(\s)|}/{|\A(\s)|}\) over 50 instances for the largest solved problem per domain.}
\label{fig:density-per-state}
\end{figure*}

\paragraphHack{Triangle Tire World with Head-start (\tireworld)}
In the original Triangle Tire World domain~\cite{Little2007:PlannerVsReplanner,Buffet2008:ippc}, the agent is provided a map of locations in a triangular layout with corners \(A, B, C\), as in \cref{fig:tireworld}.
The agent's task is to travel from corner \(A\) to \(C\), but it gets a flat tyre with probability $0.5$ every time it moves.
Once the car gets a flat tyre, it must change the tyre if a spare is available, otherwise no action is available and the goal is no longer reachable.
The agent can only store one spare at a time, and it can only obtain spare tyres in select locations (circles in \cref{fig:tireworld}).
A shortcoming of this domain is that its difficulty scales exponentially, so it may be easy to solve problem \(n\) and impractical for \(n+1\).
For this reason we extend the domain by allowing the agent a head-start, that is, its starting location may be anywhere along the edge \(AB\), for instance, on problem 2 these are locations 1-1, 2-1, \dots, 5-1 (\cref{fig:tireworld}~(right)).
Let \(\tireworld(n, d)\) denote an instance of Triangle Tire World with Head-start where \(n\) is the problem size and \(d\) denotes the distance between the agent's starting location and corner \(B\).
When \(d = 2n\), we obtain the original problem of size \(n\) and reducing \(d\) makes the problem easier until we reach \(d = 1\), the easiest variant.
Experiments using \lrtdp and \hROC suggest the following relation in terms of CPU time: \(\tireworld(n+1, 2n-3) \leq \tireworld(n, 2n) \leq \tireworld(n+1, 2n-2)\).
Therefore, for each size $n$, we consider 5 problems: $\tireworld(n, 2n-4), \dots, \tireworld(n, 2n)$.

\paragraphHack{Probabilistic Blocks World (\bw)~\cite{Buffet2008:ippc}}
As in the deterministic Blocks World from IPC, the agent is tasked with arranging blocks on a table into a particular configuration with actions to pick up blocks, put them down, or stack them.
The probabilistic version adds a 0.25 probability to each action that the handled block falls onto the table.
Furthermore, actions are added that allow the agent to pick up, put down, and stack a tower of three blocks; these have 0.9 probability of the whole tower falling onto the table.

\paragraphHack{Exploding Blocks World (\exbw)~\cite{Buffet2008:ippc}}
Another variation for the deterministic Blocks World but this time each block is rigged with an explosive that can detonate once and destroy the table or block immediately underneath it.
When a block is placed on the table or another block, it detonates with probability 0.4 and 0.1 respectively.
Once the table or a block has been destroyed, the agent can not interact with them anymore; therefore, if they are not in their goal position, the goal will be unreachable.

\paragraphHack{Probabilistic PARC Printer (\parc)~\cite{trevizan17:hpom}}
This domain is a probabilistic extension of the PARC Printer domain from IPC.
It models a modular printer consisting of various components and each page scheduled for printing needs to pass through multiple components in a particular order.
The goal is to optimise how each page is directed through the different components to satisfy the printing requirements.
With probability 0.1, a component jams ruining the relevant page and forcing it to be reprinted.
The domain comes in two flavours: with repair (\parcr), where jammed components can be repaired and then used again; and without repair (\parcn), where jammed components remain unusable.

Our code and benchmarks are available at \citet{Schmalz2024:CgilaoSourceCode}.
We now present a summary of our findings.

\paragraphHack{What is the best planner and heuristic combination?}A common metric to evaluate planners is \textit{coverage}, i.e., the total number of instances solved in a given amount of time, thus larger coverage is better.
\Cref{fig:cumulative-planner-and-heuristic} (left) shows the coverage of each combination of planner and heuristic as a function of time.
The top three combinations and their total coverages are \cgilaoH{roc} (1300), \lrtdpH{roc} (1202), and \ilaoH{roc} (1200).Note that \cgilaoH{roc} and \lrtdpH{roc} alternate in the top spot up to 220 seconds and \cgilaoH{roc} has the best coverage after that until the experiment cutoff.
Moreover, from 300 seconds onwards, \cgilaoH{roc}'s lead varies from 50 to 243 instances.
We present a breakdown of coverage per domain in \cref{tab:domain-coverage}.
For each domain considered, \cgilaoH{roc} reaches the highest coverage over other planners and heuristics.
For all three heuristics considered, \cgilao also obtains the highest coverage in all domains against the other planners for the same heuristic.
In \cref{tab:speedups}, we present the minimum and maximum speedup per domain of \cgilao over the other planners for \hROC.
The speedups w.r.t. \ilaoH{roc} vary from \(0.9\times\) (i.e., 11\% slower) in the largest \parcr problem and \(3\times\) in problem \#9 of \exbw.
Against \lrtdpH{roc}, the speedups vary from \(0.9\times\) in problem \#8 of \exbw to \(8.4\times\) in \parcn problem s4-c3.
For the performance of each planner and heuristic per problem, see \cref{sec:appendix}.

\paragraphHack{How many actions can \cgilao ignore?}
To answer this question, we look at the density of states \s, defined as \({|\partactions(\s)|}/{|\A(\s)|}\), in the final partial SSP of \cgilaoH{roc}.
\Cref{fig:density-per-state} shows, for each domain, the cumulative plot of density, i.e., how many states contain up to and including a given proportion of their applicable actions.
In all instances, at least one third of the states contains at most 50\% of the applicable actions.
The density is high for \tireworld because many states only have one applicable action.
The density is also high for \exbw because heuristics are comparatively weak for this domain.
For the other domains the results are much stronger: between 56\% and 75\% of states contain at most 50\% of the actions.
Overall, \cgilaoH{roc} added between \(38\%\) and \(66\%\) of all possible actions in its own partial SSP, and added between \(43\%\) and \(65\%\) of \ilao's actions.

\begin{table}[t!]
\setlength{\tabcolsep}{2.5pt}
\begin{center}
{\small
\begin{tabularx}{\columnwidth}{l l r r r r r|r}
\toprule
& &         \multicolumn{1}{c}{\bw} & \multicolumn{1}{c}{\exbw} & \multicolumn{1}{c}{\parcn} & \multicolumn{1}{c}{\parcr} & \multicolumn{1}{c|}{~\tireworld~} & \textbf{Total}\\
\multicolumn{2}{l}{Num. of instances} &         300 & 250 & 300 & 250 & 200 & 1300\\
\midrule
\hROC & \cgilao    & \bestCovr{300} &  \bestCovr{250} & \bestCovr{300} &  \bestCovr{250} &  \bestCovr{200} &  \bestCovr{1300} \\
      & \ilao      & \bestCovr{300} &   200           & \bestCovr{300} &  \bestCovr{250} &             150 &            1200 \\
      & \lrtdp     & 257            &  \bestCovr{250} & \bestCovr{300} &             200 &             195 &            1202 \\
\midrule
\hLMcut  & \cgilao & 150            &  \bestCovr{250} & \bestCovr{300} &             200 &             150 &            1030 \\
         & \ilao   & 150            &   200           & \bestCovr{300} &             200 &             140 &             990 \\
         & \lrtdp  &   0            &   200           & \bestCovr{300} &              50 &             149 &             699 \\
\midrule
\hMax & \cgilao    & 150            &   200           &            150 &               0 &             161 &             661 \\
      & \ilao      & 150            &   150           &            150 &               0 &             150 &             600 \\
      & \lrtdp     & 150            &   200           &            150 &               0 &             150 &             650 \\
\bottomrule
\end{tabularx}
}
\end{center}
\caption{Coverage per domain. Best coverage for each domain (column) in bold.}
\label{tab:domain-coverage}
\end{table}

\begin{table}[t!]
\setlength{\tabcolsep}{2.5pt}
\begin{center}

{\small
\setlength{\tabcolsep}{5pt}
\begin{tabularx}{\columnwidth}{c c c c c c}
\toprule
       & \bw        & \exbw      & \parcr      & \parcn     & \tireworld \\
\midrule
\ilaoH{roc}  & 1.1--1.4   & 1.0--3.0   & 0.9--1.3    & 1.3--2.2   & 1.4--1.6   \\
\lrtdpH{roc} & 1.3--2.0   & 0.9--2.9   & 2.0--8.4    & 0.7--0.9   & 1.2--1.6   \\
\bottomrule
\end{tabularx}}
\end{center}
\caption{
\cgilaoH{roc}'s speed-up over \ilaoH{roc} and \lrtdpH{roc}. The speed-ups only consider problems solved by both algorithms.}\label{tab:speedups}
\end{table}

\paragraphHack{Are \qvalues being saved?}\cgilao can save \qvalue computations by ignoring inactive actions, but at the cost of computing additional \qvalues in its separation oracle.
The cumulative plot over \qvalues in \cref{fig:cumulative-planner-and-heuristic} (right) shows that the savings in \qvalues outweigh the overhead, i.e., given a budget in \qvalues computations, \cgilao is capable of solving more instances than the other planners for the same heuristic.
At their maximum coverage, \ilaoH{roc} and \lrtdpH{roc} use \(4\times\) and \(10\times\) more \qvalues than \cgilaoH{roc}, respectively.
Moreover, \cgilaoH{roc} reaches its maximum coverage of 1300 using \(1.64 \times 10^8\) \qvalues while \ilaoH{roc} and \lrtdpH{roc} only solve 1149 and 1052 instances, respectively, for the same number of \qvalues.
A similar trend is observed when using \hLMcut; however, when using \hMax, the least informative heuristic considered, \lrtdpH{max} is slightly more \qvalue efficient than \cgilaoH{max}.

\paragraphHack{What is the impact of the heuristic on \cgilao?}
Note that, in \cref{fig:cumulative-planner-and-heuristic}, as the heuristic becomes more informative, the performance gains of \cgilao over \ilao and \lrtdp increases.
To explore this trend, we use the heuristic \(\hOptPerturbedW(\s)\) defined for \(w \in [0, 1]\) as \(\V^*(\s) \cdot r\) where \(r\) is a uniform randomly selected number from \((w , 1]\), which lets us quantify how informative a heuristic is (on average) with \(w\).
The randomness in the weight $w$ ensures that the ordering of states induced by $\hOptPerturbedW$ is different from the one induced by $V^*$.
Due to the high cost of computing \(\V^*\) we only consider the smallest problems of \bw, \exbw, \parcn, and \tireworld.
Over these problems and 50 instances each, \cref{fig:perturbed-hstar} shows the mean search time and 95\% C.I. as \(w\) varies over \(0.1, 0.2, \dots, 0.9\).
Search time excludes time spent computing the heuristic.
The ratio between \cgilao's and \ilao's runtime supports that \cgilao scales better with better heuristics: the ratio starts at 86\% and decreases to 49\% and 46\% for \(w=0.5\) and \(w=0.9\) respectively.
The reason for this behaviour is that good heuristics (i.e., tighter lower bounds) prevent \cgilao from adding inactive actions to its partial SSP, resulting in more savings in \qvalue computation.
Over all values of $w$, there is no statistically significant difference between \cgilao and \lrtdp, but both offer substantial improvement over \ilao.
This suggests \lrtdp's sampling approach can more efficiently leverage the information provided the heuristics than \ilao and \cgilao bridges the gap between them, allowing a non-sampling-based planner to use the heuristics as effectively as \lrtdp.

\begin{figure}[t!]
\includegraphics[width=\linewidth]{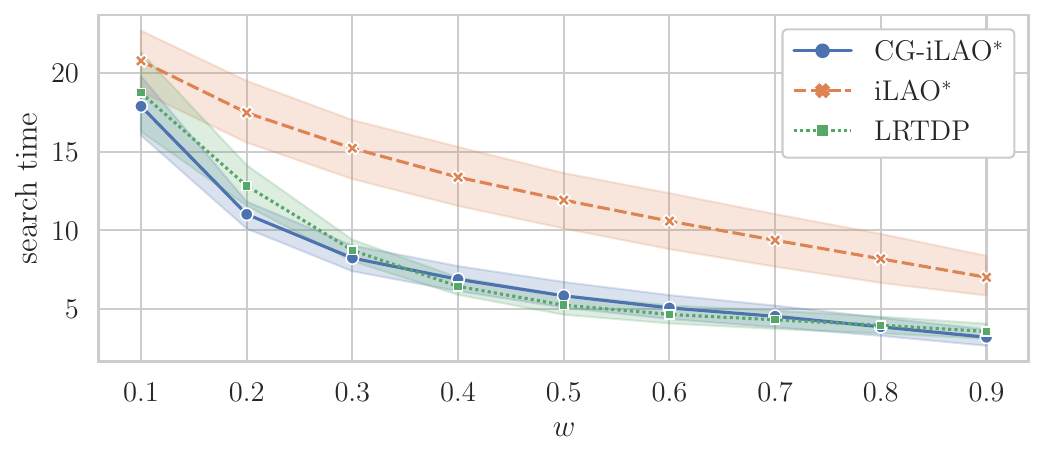}
\caption{Search time (excludes compute time for the heuristic) of algorithm with \hOptPerturbedW as \(w\) varies. We show mean and 95\% C.I. of all considered problems over 50 instances.}
\label{fig:perturbed-hstar}
\end{figure}

\section{Conclusion, Related and Future Work}

Building on existing connections between operations research and planning, we presented a new interpretation of heuristic search on SSPs as solving LPs using variable and constraint generation.
We exploit this equivalence to introduce a new and efficient separation oracle for SSPs, which enables a search algorithm to selectively add actions when they are deemed necessary to find the optimal solution.
This addresses the shortcoming of state-of-the-art algorithms that add all applicable actions during state expansion, with no mechanism for ignoring actions that will not contribute to the solution.
Using this principle, we generalised \ilao into a new optimal heuristic search algorithm \cgilao.
Empirical evaluation showed that \cgilao's ability to consider a subset of actions results in significant savings in the number of \qvalues computed, which in turn reduces the runtime of the algorithm compared to the state-of-the-art.

Regarding related work, \ref{lp:vi} has been approximated for factored MDPs to get a more compact LP, called the approximate LP (ALP).
\citet{Schuurmans2001:ConstrGenMDP} apply constraint generation to the ALP; their separation oracle has a similar condition for adding constraints as \cgilao; however, it checks for the condition naively, which is only practical on the compact ALP offered by factored MDPs, and is infeasible for SSPs.
Constraint generation lends itself well to complex planning problems where a relaxation can be efficiently solved and constraint violations by the relaxed solution can be efficiently detected, e.g., in multiagent planning~\cite{Calliess2021:ConstraintGenerationMultiAgentMIP} and metric hybrid factored planning in nonlinear domains~\cite{Say2019:ConstraintGenerationHybridPlanning}.
For POMDPs, an LP with constraint generation can be used to prune unneeded vectors from the set of vectors used to represent the value function~\cite{Walraven2017:VectorPruningPOMDP}.
In all these works, the separation oracle either naively checks all possible constraints or relies on sampling to find violations.

As future work, we aim to expand the application of \cgilao to more complex models that can benefit from our iterative method of generating applicable actions.
Models with imprecise parameters, such as MDPIPs and MDPSTs~\cite{white94mdpip,trevizan07:mdpst}, are suitable candidates for our approach since they have a \textit{minimax} semantics for the Bellman equations.
In this minimax semantics, the value function minimises the expected cost-to-go assuming that an adversary aims to maximise the cost-to-go by selecting the values of the imprecise parameters.
As a result, computing \Qsa in these models requires solving a maximisation problem; therefore, ignoring inactive actions could lead to significant improvements in performance.

Other suitable models include SSPs with \textit{PLTL constraints}~\cite{baumgartner18:pltldual,mallet21:pltlheuristics} in which both the state space and action space are augmented to keep track of constraint violations.
In these models, the concept of inactive actions can be extended to also prevent adding actions that lead to constraint violations to their partial problems.
The methods presented in this paper may also be applicable to model checking more broadly.
In particular, there has been work investigating how to use heuristics to guide the search for probabilistic reachability~\cite{Brazdil2014}, in which action elimination is applicable.

\section*{Acknowledgements}
We thank the anonymous reviewers for their feedback.
This research/project was undertaken with the assistance of resources and services from the National Computational Infrastructure (NCI), which is supported by the Australian Government.

\bibliography{references}

\clearpage

\section{Appendix}\label{sec:appendix}

We revisit the key questions of our experiments section by presenting additional results and explaining them.

\paragraphHack{What is the best planner and heuristic combination?}

We report the performance statistics of each planner and heuristic combination per benchmark problem in \crefrange{tab:bw}{tab:tireworld}.
Concretely, we report the coverage over 50 runs, and over the runs that terminated we report the mean and 95\% confidence interval associated with CPU time, \qvalues computed, and the number of calls to the heuristic.
For each problem (column) we highlight the planner and heuristic combinations with highest coverage, and among these combinations with tied highest coverage, we highlight the 95\% confidence intervals that are tied as the best (lowest) for each problem.
Formally, an interval \([x_l,x_u]\) is highlighted if there is no other interval \([y_l, y_u]\) for that problem s.t. \(y_u < x_l\).
The highlighted intervals are also known as non-dominated intervals since there is no interval that is strictly better.
If a planner and heuristic combination is missing from a table, then the planner and heuristic combination had a 0\% coverage over the relevant domain, e.g., all \ftvi combinations are missing from \parcn.

\paragraphHack{How many actions can \cgilao ignore?}

Following on from the relevant experiment in \cref{sec:experiments}, which considers the largest solved problems of each domain, we report how many actions \cgilao can ignore in \cref{tab:sparsity-over-part-ssp}.
In particular, we show the number of actions that are in \cgilao's final partial SSP, compared with the potential number of actions that can be in \cgilao's partial SSP if all applicable actions are added for all partial states, and the number of actions in \ilao's final partial SSP.
\cgilao's partial actions are significantly fewer than \ilao's total actions, i.e., \cgilao adds significantly fewer actions than \ilao, which supports our claim that, with a sufficiently informative heuristic, \cgilao considers fewer actions and thereby saves on computation.

\begin{table}[h]
\begin{center}
Num. Actions in Final Partial SSP \\
\begin{tabularx}{\columnwidth}{l r r r}
\toprule
                           & \thead{partial actions \\ \cgilao \\ \(\sum_{\s \in \partstates} |\partactions(\s)|\)} & \thead{potential actions \\ \cgilao \\ \(\sum_{\s \in \partstates} |\A(\s)|\)} & \thead{total actions \\ \ilao \\ \(\sum_{\s \in \partstates'} |\A(\s)|\)} \\ \hline
\bw                        & \(3.77 \times 10^{5}\)                           & \(9.98 \times 10^{5}\)                            & \(8.86 \times 10^{5}\)                       \\
\exbw                      & \(3.90 \times 10^{7}\)                           & \(5.98 \times 10^{7}\)                            & \(5.96 \times 10^{7}\)                       \\
\parcr                     & \(8.97 \times 10^{7}\)                           & \(2.19 \times 10^{8}\)                            & \(1.82 \times 10^{8}\)                       \\
\parcn                     & \(1.92 \times 10^{7}\)                           & \(3.91 \times 10^{7}\)                            & \(3.64 \times 10^{7}\)                       \\
\tireworld                 & \(8.05 \times 10^{7}\)                           & \(1.21 \times 10^{8}\)                            & \(1.24 \times 10^{8}\)                       \\
\bottomrule
\end{tabularx}
\end{center}
\caption{
This table shows, after running \cgilao and \ilao with \hROC, the number of \cgilao's partial actions (the actions added to \cgilao's final partial SSP), \cgilao's potential actions (the union of applicable actions over all states in \cgilao's final partial SSP), and \ilao's total actions (the actions added to \ilao's final partial SSP).
\partssp is the final partial SSP of \cgilao and \(\partssp'\) is final partial SSP of \ilao.
Each row corresponds to the largest solved problem of the domain and values are means over 50 instances.
}
\label{tab:sparsity-over-part-ssp}
\end{table}

\paragraphHack{What is the impact of the heuristic on \cgilao?}

We have already shown how the search time of \cgilao, \ilao, and \lrtdp is affected as \(w\) varies; we now show how the number of \qvalues is affected in \cref{fig:perturbed-hstar-qvalues}.
As \(w\) increases, \cgilao uses the fewest \qvalues.
This is in line with our other results, but it is important to note that, for \(w = 0.1\), \cgilao does not use fewer \qvalues than \lrtdp, which suggests that \cgilao does rely on a reasonably informative heuristic, and then scales the best as informativeness increases.

\begin{figure}[ht!]
\includegraphics[width=\linewidth]{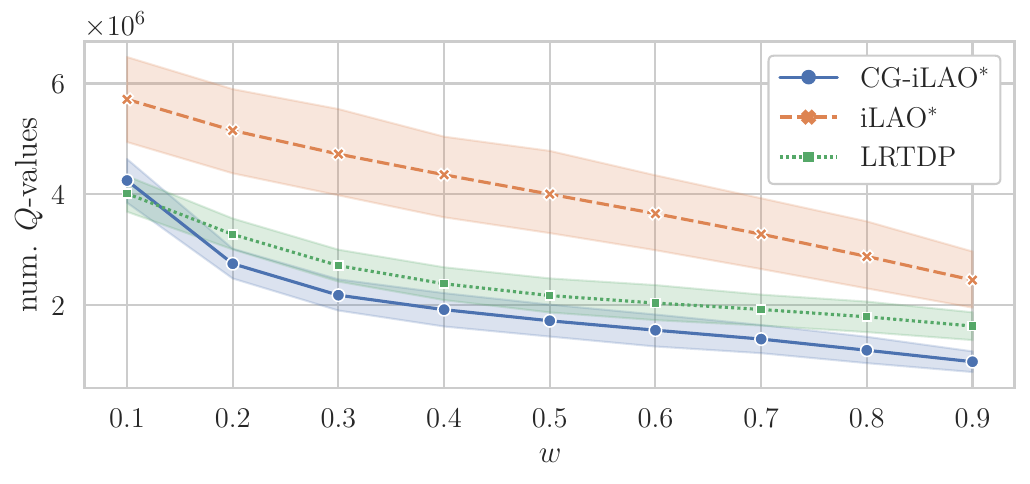}
\includegraphics[width=\linewidth]{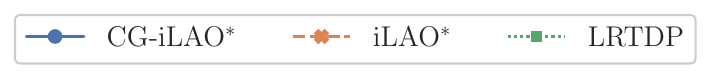}
\caption{Line plots over weight \(w\) of \hOptPerturbedW that show mean and 95\% C.I. of all considered problems over 50 instances of search time (excludes compute time for the heuristic) and the number of \qvalues computed.}
\label{fig:perturbed-hstar-qvalues}
\end{figure}

\newcommand{\covTabCaption}[1]{\caption{#1: \(A (B) [C] \langle D \rangle\) where \(A\) is the coverage out of 50 runs, \((B)\) is the mean CPU time in seconds (and the 95\% confidence interval), \([C]\) is the mean number of \qvalues computed (and the 95\% C.I.), and \(\langle D \rangle\) is the mean number of calls to the heuristic (and the 95\% C.I.). Values for $B$, $C$ and $D$ only consider successful runs.}}

\begin{landscape}
\begin{table}[]
\centering
\resizebox{\textwidth}{!}{%
\begin{tabular}{|l|l|l|l|l|l|l|l|}
\hline
algorithm & heuristic & 8-24967                                                                                               & 8-23171                                                                                                  & 8-25241                                                                                                   & 10-14262                                                                                         & 10-19475                                                                                           & 12-19848                                                                                                 \\ \hline
\cgilao   & \hROC     & \thead{\bestCovr{50} \\ \(\bestMean{(13.3 \pm 0.060)}\) \\ \(\bestMean{[14,543 \pm 175]}\) \\ \(\langle1,399 \pm 15.3\rangle\)}        & \thead{\bestCovr{50} \\ \(\bestMean{(13.7 \pm 0.082)}\) \\ \([22,432 \pm 920]\) \\ \(\bestMean{\langle1,982 \pm 81.5\rangle}\)}           & \thead{\bestCovr{50} \\ \(\bestMean{(13.1 \pm 0.078)}\) \\ \([8,844 \pm 1,019]\) \\ \(\langle1,158 \pm 114\rangle\)}            & \thead{\bestCovr{50} \\ \(\bestMean{(86.5 \pm 0.440)}\) \\ \(\bestMean{[8,781 \pm 834]}\) \\ \(\bestMean{\langle912 \pm 89.7\rangle}\)}      & \thead{\bestCovr{50} \\ \(\bestMean{(87.6 \pm 0.396)}\) \\ \(\bestMean{[15,269 \pm 1,968]}\) \\ \(\bestMean{\langle2,229 \pm 251\rangle}\)}    & \thead{\bestCovr{50} \\ \(\bestMean{(445 \pm 3.56)}\) \\ \(\bestMean{[79,771 \pm 13,664]}\) \\ \(\bestMean{\langle11,029 \pm 1,890\rangle}\)}        \\ \hline
\ftvi     & \hROC     & \thead{\bestCovr{50} \\ \((15.7 \pm 0.086)\) \\ \([16,134 \pm 248]\) \\ \(\langle5,351 \pm 98.9\rangle\)}        & \thead{\bestCovr{50} \\ \((15.5 \pm 0.175)\) \\ \(\bestMean{[17,695 \pm 829]}\) \\ \(\langle5,081 \pm 271\rangle\)}            & \thead{\bestCovr{50} \\ \((14.0 \pm 0.159)\) \\ \(\bestMean{[6,687 \pm 861]}\) \\ \(\langle2,314 \pm 248\rangle\)}              & \thead{\bestCovr{50} \\ \((91.4 \pm 0.664)\) \\ \([12,162 \pm 717]\) \\ \(\langle2,891 \pm 282\rangle\)}    & \thead{\bestCovr{50} \\ \((93.8 \pm 0.821)\) \\ \(\bestMean{[16,122 \pm 2,219]}\) \\ \(\langle6,170 \pm 723\rangle\)}    & \thead{\bestCovr{50} \\ \((523 \pm 9.09)\) \\ \([123,339 \pm 13,841]\) \\ \(\langle52,385 \pm 5,382\rangle\)}       \\ \hline
\ilao     & \hROC     & \thead{\bestCovr{50} \\ \((14.8 \pm 0.070)\) \\ \([26,575 \pm 1,122]\) \\ \(\bestMean{\langle1,234 \pm 8.92\rangle}\)}      & \thead{\bestCovr{50} \\ \((15.2 \pm 0.073)\) \\ \([49,315 \pm 2,779]\) \\ \(\bestMean{\langle1,914 \pm 76.1\rangle}\)}         & \thead{\bestCovr{50} \\ \((14.5 \pm 0.078)\) \\ \([20,672 \pm 2,543]\) \\ \(\bestMean{\langle889 \pm 94.4\rangle}\)}            & \thead{\bestCovr{50} \\ \((97.5 \pm 0.560)\) \\ \([36,477 \pm 4,751]\) \\ \(\bestMean{\langle819 \pm 43.4\rangle}\)}   & \thead{\bestCovr{50} \\ \((98.0 \pm 0.441)\) \\ \([41,180 \pm 5,678]\) \\ \(\bestMean{\langle1,888 \pm 206\rangle}\)}    & \thead{\bestCovr{50} \\ \((507 \pm 3.43)\) \\ \([260,141 \pm 33,101]\) \\ \(\bestMean{\langle13,659 \pm 1,739\rangle}\)}       \\ \hline
\lrtdp    & \hROC     & \thead{\bestCovr{50} \\ \((24.0 \pm 0.283)\) \\ \([102,730 \pm 2,772]\) \\ \(\langle19,342 \pm 522\rangle\)}     & \thead{\bestCovr{50} \\ \((25.6 \pm 0.424)\) \\ \([128,979 \pm 5,024]\) \\ \(\langle22,313 \pm 811\rangle\)}        & \thead{\bestCovr{50} \\ \((26.2 \pm 1.18)\) \\ \([144,314 \pm 14,547]\) \\ \(\langle22,941 \pm 2,053\rangle\)}       & \thead{\bestCovr{50} \\ \((109 \pm 1.32)\) \\ \([55,781 \pm 5,392]\) \\ \(\langle14,774 \pm 1,470\rangle\)} & \thead{\bestCovr{50} \\ \((150 \pm 5.61)\) \\ \([248,434 \pm 23,864]\) \\ \(\langle58,918 \pm 5,093\rangle\)} & \thead{7 \\ \((1,301 \pm 258)\) \\ \([2,304,921 \pm 724,145]\) \\ \(\langle532,833 \pm 159,334\rangle\)} \\ \hline
\cgilao   & \hLMcut   & \thead{\bestCovr{50} \\ \((132 \pm 7.05)\) \\ \([780,075 \pm 60,269]\) \\ \(\langle35,591 \pm 2,054\rangle\)}    & \thead{\bestCovr{50} \\ \((573 \pm 21.5)\) \\ \([3,609,948 \pm 220,483]\) \\ \(\langle129,509 \pm 5,591\rangle\)}   & \thead{\bestCovr{50} \\ \((512 \pm 29.6)\) \\ \([2,787,033 \pm 277,267]\) \\ \(\langle106,939 \pm 6,900\rangle\)}    &                                                                                                  &                                                                                                    &                                                                                                          \\ \hline
\ilao     & \hLMcut   & \thead{\bestCovr{50} \\ \((121 \pm 6.59)\) \\ \([851,211 \pm 54,597]\) \\ \(\langle31,891 \pm 1,871\rangle\)}    & \thead{\bestCovr{50} \\ \((545 \pm 21.0)\) \\ \([3,955,151 \pm 217,566]\) \\ \(\langle116,921 \pm 4,809\rangle\)}   & \thead{\bestCovr{50} \\ \((478 \pm 26.9)\) \\ \([2,948,487 \pm 241,209]\) \\ \(\langle99,287 \pm 6,219\rangle\)}     &                                                                                                  &                                                                                                    &                                                                                                          \\ \hline
\cgilao   & \hMax     & \thead{\bestCovr{50} \\ \((302 \pm 3.25)\) \\ \([16,176,234 \pm 22,293]\) \\ \(\langle459,032 \pm 180\rangle\)}  & \thead{\bestCovr{50} \\ \((1,615 \pm 14.4)\) \\ \([86,204,696 \pm 69,607]\) \\ \(\langle922,272 \pm 0.000\rangle\)} & \thead{\bestCovr{50} \\ \((1,688 \pm 10.3)\) \\ \([89,642,935 \pm 187,375]\) \\ \(\langle922,271 \pm 0.000\rangle\)} &                                                                                                  &                                                                                                    &                                                                                                          \\ \hline
\ilao     & \hMax     & \thead{\bestCovr{50} \\ \((240 \pm 3.19)\) \\ \([18,475,207 \pm 83,150]\) \\ \(\langle378,670 \pm 445\rangle\)}  & \thead{\bestCovr{50} \\ \((743 \pm 8.73)\) \\ \([60,577,420 \pm 130,400]\) \\ \(\langle805,407 \pm 188\rangle\)}    & \thead{\bestCovr{50} \\ \((819 \pm 8.47)\) \\ \([66,960,805 \pm 29,399]\) \\ \(\langle839,082 \pm 133\rangle\)}      &                                                                                                  &                                                                                                    &                                                                                                          \\ \hline
\lrtdp    & \hMax     & \thead{\bestCovr{50} \\ \((567 \pm 2.48)\) \\ \([33,537,233 \pm 30,476]\) \\ \(\langle922,203 \pm 2.75\rangle\)} & \thead{\bestCovr{50} \\ \((748 \pm 4.41)\) \\ \([49,482,874 \pm 26,003]\) \\ \(\langle922,273 \pm 0.000\rangle\)}   & \thead{\bestCovr{50} \\ \((863 \pm 6.07)\) \\ \([58,962,949 \pm 50,989]\) \\ \(\langle922,272 \pm 0.000\rangle\)}    &                                                                                                  &                                                                                                    &                                                                                                          \\ \hline
\end{tabular}%
}
\covTabCaption{\bw}
\label{tab:bw}
\end{table}
\end{landscape}

\begin{landscape}
\begin{table}[]
\centering
\resizebox{\textwidth}{!}{%
\begin{tabular}{|l|l|l|l|l|l|l|}
\hline
algorithm & heuristic & p07-n7-N9-s7                                                                                          & p08-n8-N10-s8                                                                                                & p09-n9-N11-s9                                                                                           & p10-n10-N12-s10                                                                                          & p11-n11-N13-s11                                                                                              \\ \hline
\cgilao   & \hROC     & \thead{\bestCovr{50} \\ \((0.585 \pm 0.031)\) \\ \([3,289 \pm 360]\) \\ \(\langle1,406 \pm 130\rangle\)}         & \thead{\bestCovr{50} \\ \((68.8 \pm 1.70)\) \\ \([8,551,674 \pm 24,812]\) \\ \(\langle164,822 \pm 1,461\rangle\)}       & \thead{\bestCovr{50} \\ \((7.09 \pm 0.086)\) \\ \(\bestMean{[458,161 \pm 2,965]}\) \\ \(\bestMean{\langle21,503 \pm 220\rangle}\)}       & \thead{\bestCovr{50} \\ \((299 \pm 9.21)\) \\ \([75,922,919 \pm 228,249]\) \\ \(\langle214,709 \pm 1,046\rangle\)}  & \thead{\bestCovr{50} \\ \(\bestMean{(433 \pm 2.36)}\) \\ \([11,183,164 \pm 33,827]\) \\ \(\langle1,444,012 \pm 2,101\rangle\)}     \\ \hline
\ftvi     & \hROC     & \thead{\bestCovr{50} \\ \((1.15 \pm 0.127)\) \\ \([4,323 \pm 687]\) \\ \(\langle3,929 \pm 564\rangle\)}          &                                                                                                              &                                                                                                         &                                                                                                          &                                                                                                              \\ \hline
\ilao     & \hROC     & \thead{\bestCovr{50} \\ \((0.601 \pm 0.026)\) \\ \([9,971 \pm 571]\) \\ \(\langle1,434 \pm 112\rangle\)}         & \thead{\bestCovr{50} \\ \((200 \pm 1.86)\) \\ \([31,838,096 \pm 38,023]\) \\ \(\langle160,179 \pm 1,431\rangle\)}       & \thead{\bestCovr{50} \\ \((21.3 \pm 0.802)\) \\ \([2,689,211 \pm 133,082]\) \\ \(\langle27,797 \pm 284\rangle\)}   &                                                                                                          & \thead{\bestCovr{50} \\ \((490 \pm 2.96)\) \\ \([26,675,612 \pm 33,185]\) \\ \(\langle1,465,487 \pm 1,379\rangle\)}     \\ \hline
\lrtdp    & \hROC     & \thead{\bestCovr{50} \\ \((1.16 \pm 0.092)\) \\ \([11,266 \pm 1,480]\) \\ \(\langle3,904 \pm 376\rangle\)}       & \thead{\bestCovr{50} \\ \(\bestMean{(62.5 \pm 1.01)}\) \\ \([3,048,449 \pm 95,535]\) \\ \(\langle228,649 \pm 3,906\rangle\)}       & \thead{\bestCovr{50} \\ \((19.9 \pm 0.270)\) \\ \([1,688,157 \pm 20,216]\) \\ \(\langle63,101 \pm 985\rangle\)}    & \thead{\bestCovr{50} \\ \((136 \pm 1.80)\) \\ \(\bestMean{[8,741,452 \pm 62,416]}\) \\ \(\langle428,008 \pm 5,579\rangle\)}    & \thead{\bestCovr{50} \\ \((1,272 \pm 26.3)\) \\ \([20,222,574 \pm 270,633]\) \\ \(\langle4,334,206 \pm 68,377\rangle\)} \\ \hline
\cgilao   & \hLMcut   & \thead{\bestCovr{50} \\ \(\bestMean{(0.136 \pm 0.006)}\) \\ \(\bestMean{[1,672 \pm 69.7]}\) \\ \(\bestMean{\langle868 \pm 34.1\rangle}\)}         & \thead{\bestCovr{50} \\ \((65.8 \pm 1.39)\) \\ \([9,020,624 \pm 21,098]\) \\ \(\bestMean{\langle137,032 \pm 1,593\rangle}\)}       & \thead{\bestCovr{50} \\ \((8.88 \pm 0.082)\) \\ \([598,012 \pm 3,912]\) \\ \(\langle22,187 \pm 320\rangle\)}       & \thead{\bestCovr{50} \\ \((347 \pm 8.46)\) \\ \([67,704,563 \pm 205,137]\) \\ \(\bestMean{\langle208,048 \pm 845\rangle}\)}    & \thead{\bestCovr{50} \\ \((1,016 \pm 4.25)\) \\ \(\bestMean{[11,026,216 \pm 33,388]}\) \\ \(\langle1,268,011 \pm 2,443\rangle\)}   \\ \hline
\ftvi     & \hLMcut   & \thead{\bestCovr{50} \\ \((0.561 \pm 0.082)\) \\ \([3,828 \pm 615]\) \\ \(\langle3,431 \pm 506\rangle\)}         &                                                                                                              &                                                                                                         &                                                                                                          &                                                                                                              \\ \hline
\ilao     & \hLMcut   & \thead{\bestCovr{50} \\ \((0.233 \pm 0.021)\) \\ \([8,635 \pm 453]\) \\ \(\langle1,297 \pm 101\rangle\)}         & \thead{\bestCovr{50} \\ \((229 \pm 3.25)\) \\ \([35,341,647 \pm 45,652]\) \\ \(\bestMean{\langle137,593 \pm 1,550\rangle}\)}       & \thead{\bestCovr{50} \\ \((21.8 \pm 0.678)\) \\ \([2,581,881 \pm 102,361]\) \\ \(\langle27,040 \pm 285\rangle\)}   &                                                                                                          & \thead{\bestCovr{50} \\ \((1,111 \pm 6.48)\) \\ \([25,842,856 \pm 43,427]\) \\ \(\bestMean{\langle1,263,692 \pm 837\rangle}\)}     \\ \hline
\lrtdp    & \hLMcut   & \thead{\bestCovr{50} \\ \((0.485 \pm 0.047)\) \\ \([5,561 \pm 716]\) \\ \(\langle2,863 \pm 280\rangle\)}         & \thead{\bestCovr{50} \\ \((91.9 \pm 1.23)\) \\ \(\bestMean{[2,858,397 \pm 56,354]}\) \\ \(\langle240,751 \pm 3,305\rangle\)}       & \thead{\bestCovr{50} \\ \((35.7 \pm 0.613)\) \\ \([1,704,587 \pm 20,139]\) \\ \(\langle65,962 \pm 1,098\rangle\)}  & \thead{\bestCovr{50} \\ \((392 \pm 5.43)\) \\ \([9,892,329 \pm 59,686]\) \\ \(\langle516,447 \pm 7,122\rangle\)}    &                                                                                                              \\ \hline
\cgilao   & \hMax     & \thead{\bestCovr{50} \\ \((1.69 \pm 0.011)\) \\ \([284,993 \pm 2,474]\) \\ \(\langle43,323 \pm 242\rangle\)}     & \thead{\bestCovr{50} \\ \((78.5 \pm 1.96)\) \\ \([18,188,194 \pm 35,225]\) \\ \(\langle819,653 \pm 1,842\rangle\)}      & \thead{\bestCovr{50} \\ \(\bestMean{(4.00 \pm 0.051)}\) \\ \([966,920 \pm 5,282]\) \\ \(\langle43,816 \pm 443\rangle\)}       & \thead{\bestCovr{50} \\ \((308 \pm 11.5)\) \\ \([107,680,890 \pm 50,168]\) \\ \(\langle282,810 \pm 945\rangle\)}    &                                                                                                              \\ \hline
\ilao     & \hMax     & \thead{\bestCovr{50} \\ \((2.98 \pm 0.023)\) \\ \([837,143 \pm 4,822]\) \\ \(\langle41,336 \pm 324\rangle\)}     & \thead{\bestCovr{50} \\ \((331 \pm 5.48)\) \\ \([57,048,934 \pm 51,381]\) \\ \(\langle739,570 \pm 1,511\rangle\)}       & \thead{\bestCovr{50} \\ \((13.8 \pm 0.764)\) \\ \([2,678,145 \pm 114,072]\) \\ \(\langle42,956 \pm 410\rangle\)}   &                                                                                                          &                                                                                                              \\ \hline
\lrtdp    & \hMax     & \thead{\bestCovr{50} \\ \((6.59 \pm 0.169)\) \\ \([779,021 \pm 18,621]\) \\ \(\langle175,517 \pm 3,824\rangle\)} & \thead{\bestCovr{50} \\ \((89.1 \pm 0.880)\) \\ \([19,435,368 \pm 123,222]\) \\ \(\langle1,349,405 \pm 25,971\rangle\)} & \thead{\bestCovr{50} \\ \((14.0 \pm 0.195)\) \\ \([3,314,262 \pm 32,347]\) \\ \(\langle145,387 \pm 1,891\rangle\)} & \thead{\bestCovr{50} \\ \(\bestMean{(105 \pm 0.821)}\) \\ \([19,713,845 \pm 112,489]\) \\ \(\langle842,756 \pm 9,501\rangle\)} &                                                                                                              \\ \hline
\end{tabular}%
}
\covTabCaption{\exbw}
\label{tab:exbw}
\end{table}
\end{landscape}

\begin{landscape}
\begin{table}[]
\centering
\resizebox{\textwidth}{!}{%
\begin{tabular}{|l|l|l|l|l|l|l|l|}
\hline
algorithm & heuristic & s4-c1                                                                                                   & s4-c2                                                                                                      & s4-c3                                                                                                           & s5-c1                                                                                                     & s5-c2                                                                                                      & s5-c3                                                                                                         \\ \hline
\cgilao   & \hROC     & \thead{\bestCovr{50} \\ \((10.5 \pm 0.155)\) \\ \([1,013,976 \pm 15,173]\) \\ \(\langle31,350 \pm 447\rangle\)}    & \thead{\bestCovr{50} \\ \((16.1 \pm 0.256)\) \\ \([1,715,321 \pm 20,139]\) \\ \(\langle42,106 \pm 761\rangle\)}       & \thead{\bestCovr{50} \\ \((22.1 \pm 0.403)\) \\ \([3,022,500 \pm 43,724]\) \\ \(\langle54,854 \pm 1,271\rangle\)}          & \thead{\bestCovr{50} \\ \((114 \pm 1.38)\) \\ \([12,324,266 \pm 114,576]\) \\ \(\langle307,166 \pm 3,809\rangle\)}   & \thead{\bestCovr{50} \\ \((180 \pm 3.49)\) \\ \([20,993,881 \pm 208,715]\) \\ \(\langle404,594 \pm 7,231\rangle\)}    & \thead{\bestCovr{50} \\ \((244 \pm 4.14)\) \\ \([37,055,726 \pm 466,292]\) \\ \(\langle515,254 \pm 12,255\rangle\)}      \\ \hline
\ilao     & \hROC     & \thead{\bestCovr{50} \\ \((14.0 \pm 0.194)\) \\ \([3,044,522 \pm 74,676]\) \\ \(\bestMean{\langle27,845 \pm 269\rangle}\)}    & \thead{\bestCovr{50} \\ \((24.0 \pm 0.389)\) \\ \([5,344,883 \pm 121,538]\) \\ \(\bestMean{\langle38,451 \pm 543\rangle}\)}      & \thead{\bestCovr{50} \\ \((39.8 \pm 0.888)\) \\ \([9,597,755 \pm 277,641]\) \\ \(\bestMean{\langle50,201 \pm 1,045\rangle}\)}         & \thead{\bestCovr{50} \\ \((184 \pm 2.73)\) \\ \([43,507,677 \pm 1,056,857]\) \\ \(\bestMean{\langle286,297 \pm 2,390\rangle}\)} & \thead{\bestCovr{50} \\ \((321 \pm 6.36)\) \\ \([73,968,382 \pm 1,836,330]\) \\ \(\bestMean{\langle378,071 \pm 5,873\rangle}\)}  & \thead{\bestCovr{50} \\ \((525 \pm 8.69)\) \\ \([127,869,617 \pm 3,713,867]\) \\ \(\bestMean{\langle478,778 \pm 9,732\rangle}\)}    \\ \hline
\lrtdp    & \hROC     & \thead{\bestCovr{50} \\ \(\bestMean{(8.78 \pm 0.141)}\) \\ \(\bestMean{[436,531 \pm 3,058]}\) \\ \(\langle31,717 \pm 561\rangle\)}       & \thead{\bestCovr{50} \\ \(\bestMean{(13.1 \pm 0.270)}\) \\ \(\bestMean{[585,888 \pm 3,071]}\) \\ \(\langle41,835 \pm 942\rangle\)}          & \thead{\bestCovr{50} \\ \(\bestMean{(16.4 \pm 0.372)}\) \\ \(\bestMean{[680,237 \pm 4,190]}\) \\ \(\langle53,681 \pm 1,522\rangle\)}             & \thead{\bestCovr{50} \\ \(\bestMean{(100 \pm 1.23)}\) \\ \(\bestMean{[5,261,132 \pm 41,743]}\) \\ \(\langle310,181 \pm 4,941\rangle\)}     & \thead{\bestCovr{50} \\ \(\bestMean{(150 \pm 2.32)}\) \\ \(\bestMean{[7,098,114 \pm 43,550]}\) \\ \(\langle398,327 \pm 8,576\rangle\)}      & \thead{\bestCovr{50} \\ \(\bestMean{(184 \pm 3.72)}\) \\ \(\bestMean{[8,273,497 \pm 40,891]}\) \\ \(\bestMean{\langle496,371 \pm 13,701\rangle}\)}        \\ \hline
\cgilao   & \hLMcut   & \thead{\bestCovr{50} \\ \((10.7 \pm 0.070)\) \\ \([1,608,662 \pm 17,455]\) \\ \(\langle46,829 \pm 436\rangle\)}    & \thead{\bestCovr{50} \\ \((19.4 \pm 0.138)\) \\ \([3,367,791 \pm 40,108]\) \\ \(\langle77,150 \pm 782\rangle\)}       & \thead{\bestCovr{50} \\ \((33.1 \pm 0.196)\) \\ \([5,962,957 \pm 81,595]\) \\ \(\langle138,353 \pm 1,487\rangle\)}         & \thead{\bestCovr{50} \\ \((141 \pm 0.855)\) \\ \([19,536,893 \pm 198,117]\) \\ \(\langle484,752 \pm 4,355\rangle\)}  & \thead{\bestCovr{50} \\ \((268 \pm 2.53)\) \\ \([43,716,409 \pm 423,614]\) \\ \(\langle829,264 \pm 8,947\rangle\)}    & \thead{\bestCovr{50} \\ \((477 \pm 5.28)\) \\ \([75,760,732 \pm 884,923]\) \\ \(\langle1,476,484 \pm 17,030\rangle\)}    \\ \hline
\ilao     & \hLMcut   & \thead{\bestCovr{50} \\ \((16.5 \pm 0.198)\) \\ \([4,649,445 \pm 114,510]\) \\ \(\langle39,767 \pm 379\rangle\)}   & \thead{\bestCovr{50} \\ \((36.5 \pm 0.499)\) \\ \([10,851,237 \pm 231,976]\) \\ \(\langle64,040 \pm 835\rangle\)}     & \thead{\bestCovr{50} \\ \((61.8 \pm 1.06)\) \\ \([17,429,097 \pm 478,163]\) \\ \(\langle109,311 \pm 1,250\rangle\)}        & \thead{\bestCovr{50} \\ \((229 \pm 3.64)\) \\ \([59,864,905 \pm 1,529,872]\) \\ \(\langle405,784 \pm 3,557\rangle\)} & \thead{\bestCovr{50} \\ \((539 \pm 8.11)\) \\ \([148,174,993 \pm 3,339,338]\) \\ \(\langle658,100 \pm 8,781\rangle\)} & \thead{\bestCovr{50} \\ \((941 \pm 15.5)\) \\ \([238,099,853 \pm 6,527,230]\) \\ \(\langle1,121,729 \pm 13,299\rangle\)} \\ \hline
\lrtdp    & \hLMcut   & \thead{\bestCovr{50} \\ \((11.6 \pm 0.165)\) \\ \([932,940 \pm 5,927]\) \\ \(\langle68,166 \pm 1,087\rangle\)}     & \thead{\bestCovr{50} \\ \((20.9 \pm 0.315)\) \\ \([1,490,562 \pm 8,226]\) \\ \(\langle116,119 \pm 2,166\rangle\)}     & \thead{\bestCovr{50} \\ \((32.0 \pm 0.404)\) \\ \([2,113,526 \pm 10,839]\) \\ \(\langle177,797 \pm 3,138\rangle\)}         & \thead{\bestCovr{50} \\ \((156 \pm 1.72)\) \\ \([11,634,002 \pm 108,431]\) \\ \(\langle702,833 \pm 10,416\rangle\)}  & \thead{\bestCovr{50} \\ \((269 \pm 3.27)\) \\ \([18,809,583 \pm 146,182]\) \\ \(\langle1,182,469 \pm 23,053\rangle\)} & \thead{\bestCovr{50} \\ \((432 \pm 5.53)\) \\ \([26,274,405 \pm 171,814]\) \\ \(\langle1,808,411 \pm 33,918\rangle\)}    \\ \hline
\cgilao   & \hMax     & \thead{\bestCovr{50} \\ \((107 \pm 1.18)\) \\ \([45,597,493 \pm 83,407]\) \\ \(\langle639,153 \pm 73.7\rangle\)}   & \thead{\bestCovr{50} \\ \((374 \pm 5.78)\) \\ \([132,743,888 \pm 220,135]\) \\ \(\langle1,940,914 \pm 351\rangle\)}   & \thead{\bestCovr{50} \\ \((610 \pm 11.1)\) \\ \([203,454,719 \pm 410,922]\) \\ \(\langle2,773,558 \pm 1,585\rangle\)}      &                                                                                                           &                                                                                                            &                                                                                                               \\ \hline
\ilao     & \hMax     & \thead{\bestCovr{50} \\ \((217 \pm 3.22)\) \\ \([99,246,994 \pm 1,488,977]\) \\ \(\langle618,593 \pm 142\rangle\)} & \thead{\bestCovr{50} \\ \((790 \pm 14.3)\) \\ \([317,087,417 \pm 6,424,885]\) \\ \(\langle1,787,826 \pm 447\rangle\)} & \thead{\bestCovr{50} \\ \((1,425 \pm 38.8)\) \\ \([448,028,105 \pm 13,746,351]\) \\ \(\langle2,229,375 \pm 1,267\rangle\)} &                                                                                                           &                                                                                                            &                                                                                                               \\ \hline
\lrtdp    & \hMax     & \thead{\bestCovr{50} \\ \((39.5 \pm 0.419)\) \\ \([17,779,098 \pm 21,526]\) \\ \(\langle588,247 \pm 45.9\rangle\)} & \thead{\bestCovr{50} \\ \((254 \pm 2.51)\) \\ \([126,641,023 \pm 252,570]\) \\ \(\langle1,627,845 \pm 75.4\rangle\)}  & \thead{\bestCovr{50} \\ \((403 \pm 5.60)\) \\ \([175,415,264 \pm 251,069]\) \\ \(\langle2,593,447 \pm 1,011\rangle\)}      &                                                                                                           &                                                                                                            &                                                                                                               \\ \hline
\end{tabular}%
}
\covTabCaption{\parcn}
\label{tab:parcn}
\end{table}
\end{landscape}

\begin{landscape}
\begin{table}[]
\centering
\resizebox{\textwidth}{!}{%
\begin{tabular}{|l|l|l|l|l|l|l|l|}
\hline
algorithm & heuristic & s4-c1                                                                                                  & s4-c2                                                                                                    & s4-c3                                                                                                      & s5-c1                                                                                                        & s5-c2                                                                                                          \\ \hline
\cgilao   & \hROC     & \thead{\bestCovr{50} \\ \(\bestMean{(45.9 \pm 0.525)}\) \\ \(\bestMean{[6,228,997 \pm 20,656]}\) \\ \(\langle95,096 \pm 261\rangle\)}   & \thead{\bestCovr{50} \\ \(\bestMean{(121 \pm 1.38)}\) \\ \([14,811,715 \pm 28,578]\) \\ \(\langle208,401 \pm 175\rangle\)}     & \thead{\bestCovr{50} \\ \((181 \pm 2.27)\) \\ \([22,138,217 \pm 27,456]\) \\ \(\langle281,753 \pm 109\rangle\)}       & \thead{\bestCovr{50} \\ \(\bestMean{(583 \pm 7.32)}\) \\ \(\bestMean{[63,821,278 \pm 180,676]}\) \\ \(\langle1,061,846 \pm 1,937\rangle\)}    & \thead{\bestCovr{50} \\ \((1,657 \pm 16.4)\) \\ \(\bestMean{[161,991,609 \pm 242,532]}\) \\ \(\langle2,299,310 \pm 1,043\rangle\)} \\ \hline
\ftvi     & \hROC     &                                                                                                        & \thead{\bestCovr{50} \\ \((765 \pm 37.9)\) \\ \(\bestMean{[5,514,887 \pm 39,772]}\) \\ \(\langle235,023 \pm 1,076\rangle\)}    & \thead{\bestCovr{50} \\ \((613 \pm 69.3)\) \\ \(\bestMean{[6,767,005 \pm 49,323]}\) \\ \(\langle307,358 \pm 1,049\rangle\)}      &                                                                                                              &                                                                                                                \\ \hline
\ilao     & \hROC     & \thead{\bestCovr{50} \\ \((58.3 \pm 0.398)\) \\ \([15,636,208 \pm 88,234]\) \\ \(\bestMean{\langle84,340 \pm 123\rangle}\)}  & \thead{\bestCovr{50} \\ \((126 \pm 0.667)\) \\ \([32,377,004 \pm 137,024]\) \\ \(\bestMean{\langle183,079 \pm 223\rangle}\)}   & \thead{\bestCovr{50} \\ \(\bestMean{(168 \pm 0.749)}\) \\ \([41,961,804 \pm 133,535]\) \\ \(\langle268,261 \pm 312\rangle\)}     & \thead{\bestCovr{50} \\ \((652 \pm 4.57)\) \\ \([150,228,580 \pm 756,329]\) \\ \(\bestMean{\langle887,525 \pm 1,217\rangle}\)}     & \thead{\bestCovr{50} \\ \(\bestMean{(1,502 \pm 17.5)}\) \\ \([292,637,047 \pm 1,224,742]\) \\ \(\bestMean{\langle1,945,156 \pm 3,643\rangle}\)} \\ \hline
\lrtdp    & \hROC     & \thead{\bestCovr{50} \\ \((93.2 \pm 0.568)\) \\ \([41,042,790 \pm 152,970]\) \\ \(\langle92,102 \pm 236\rangle\)} & \thead{\bestCovr{50} \\ \((674 \pm 4.92)\) \\ \([352,166,761 \pm 1,539,678]\) \\ \(\langle196,298 \pm 270\rangle\)} & \thead{\bestCovr{50} \\ \((1,513 \pm 8.48)\) \\ \([797,662,187 \pm 3,139,450]\) \\ \(\bestMean{\langle265,844 \pm 213\rangle}\)} & \thead{\bestCovr{50} \\ \((1,335 \pm 11.3)\) \\ \([547,832,850 \pm 1,127,533]\) \\ \(\langle984,575 \pm 2,753\rangle\)} &                                                                                                                \\ \hline
\cgilao   & \hLMcut   & \thead{\bestCovr{50} \\ \((68.8 \pm 0.339)\) \\ \([9,184,226 \pm 26,200]\) \\ \(\langle148,718 \pm 270\rangle\)}  & \thead{\bestCovr{50} \\ \((212 \pm 2.82)\) \\ \([26,805,385 \pm 35,474]\) \\ \(\langle351,686 \pm 229\rangle\)}     & \thead{\bestCovr{50} \\ \((322 \pm 2.82)\) \\ \([38,522,863 \pm 94,665]\) \\ \(\langle517,017 \pm 607\rangle\)}       & \thead{\bestCovr{50} \\ \((1,003 \pm 7.77)\) \\ \([105,649,983 \pm 292,360]\) \\ \(\langle1,669,291 \pm 2,205\rangle\)} &                                                                                                                \\ \hline
\ftvi     & \hLMcut   &                                                                                                        &                                                                                                          & \thead{20 \\ \((367 \pm 116)\) \\ \([12,477,085 \pm 70,147]\) \\ \(\langle563,694 \pm 6,066\rangle\)}      &                                                                                                              &                                                                                                                \\ \hline
\ilao     & \hLMcut   & \thead{\bestCovr{50} \\ \((84.9 \pm 1.31)\) \\ \([20,392,448 \pm 117,401]\) \\ \(\langle143,780 \pm 238\rangle\)} & \thead{\bestCovr{50} \\ \((213 \pm 1.13)\) \\ \([49,478,612 \pm 188,165]\) \\ \(\langle342,211 \pm 345\rangle\)}    & \thead{\bestCovr{50} \\ \((296 \pm 2.26)\) \\ \([67,320,675 \pm 240,270]\) \\ \(\langle471,351 \pm 641\rangle\)}      & \thead{\bestCovr{50} \\ \((1,112 \pm 6.03)\) \\ \([209,913,627 \pm 967,213]\) \\ \(\langle1,614,163 \pm 2,175\rangle\)} &                                                                                                                \\ \hline
\lrtdp    & \hLMcut   & \thead{\bestCovr{50} \\ \((340 \pm 1.30)\) \\ \([119,765,590 \pm 396,502]\) \\ \(\langle190,315 \pm 712\rangle\)} &                                                                                                          &                                                                                                            &                                                                                                              &                                                                                                                \\ \hline
\end{tabular}%
}
\covTabCaption{\parcr}
\label{tab:parcr}
\end{table}
\end{landscape}

\begin{table*}[]
\centering
\resizebox{\textwidth}{!}{%
\begin{tabular}{|l|l|l|l|l|l|l|}
\hline
algorithm & heuristic & \(\tireworld(4, 8)\) (original \tireworld 4) & \(\tireworld(5, 10)\) (original \tireworld 5) & \(\tireworld(6, 8)\) & \(\tireworld(6, 9)\)                                                                                           \\ \hline
\cgilao   & \hROC     & \thead{\bestCovr{50} \\ \((14.4 \pm 0.211)\) \\ \([1,148,040 \pm 16,387]\) \\ \(\bestMean{\langle40,634 \pm 136\rangle}\)}    & \thead{\bestCovr{50} \\ \(\bestMean{(395 \pm 3.11)}\) \\ \([31,842,440 \pm 247,363]\) \\ \(\bestMean{\langle579,176 \pm 1,779\rangle}\)}      & \thead{\bestCovr{50} \\ \(\bestMean{(516 \pm 6.29)}\) \\ \([32,427,798 \pm 556,567]\) \\ \(\bestMean{\langle908,936 \pm 1,665\rangle}\)}         & \thead{\bestCovr{50} \\ \(\bestMean{(1,223 \pm 14.4)}\) \\ \(\bestMean{[80,746,215 \pm 1,061,667]}\) \\ \(\bestMean{\langle1,875,857 \pm 4,111\rangle}\)} \\ \hline
\ftvi     & \hROC     & \thead{\bestCovr{50} \\ \((42.0 \pm 0.601)\) \\ \(\bestMean{[386,702 \pm 3,356]}\) \\ \(\langle190,435 \pm 3,029\rangle\)}    & \thead{\bestCovr{50} \\ \((695 \pm 11.5)\) \\ \(\bestMean{[7,290,773 \pm 83,692]}\) \\ \(\langle2,400,870 \pm 43,908\rangle\)}     & \thead{\bestCovr{50} \\ \((1,156 \pm 31.4)\) \\ \(\bestMean{[9,429,896 \pm 120,816]}\) \\ \(\langle3,646,876 \pm 109,529\rangle\)}    &                                                                                                               \\ \hline
\ilao     & \hROC     & \thead{\bestCovr{50} \\ \((20.3 \pm 0.162)\) \\ \([2,456,523 \pm 11,941]\) \\ \(\langle50,203 \pm 295\rangle\)}    & \thead{\bestCovr{50} \\ \((627 \pm 5.06)\) \\ \([64,753,531 \pm 556,643]\) \\ \(\langle730,691 \pm 4,690\rangle\)}      & \thead{\bestCovr{50} \\ \((794 \pm 9.02)\) \\ \([63,978,036 \pm 925,163]\) \\ \(\langle1,175,526 \pm 7,985\rangle\)}       &                                                                                                               \\ \hline
\lrtdp    & \hROC     & \thead{\bestCovr{50} \\ \((23.6 \pm 0.134)\) \\ \([2,003,448 \pm 12,008]\) \\ \(\langle88,076 \pm 545\rangle\)}    & \thead{\bestCovr{50} \\ \((481 \pm 6.70)\) \\ \([49,385,315 \pm 285,830]\) \\ \(\langle1,048,156 \pm 7,230\rangle\)}    & \thead{\bestCovr{50} \\ \((719 \pm 7.32)\) \\ \([58,413,438 \pm 267,244]\) \\ \(\langle1,563,223 \pm 13,711\rangle\)}      & \thead{45 \\ \((1,667 \pm 16.6)\) \\ \([140,361,933 \pm 662,188]\) \\ \(\langle3,195,700 \pm 27,373\rangle\)} \\ \hline
\cgilao   & \hLMcut   & \thead{\bestCovr{50} \\ \((14.9 \pm 0.380)\) \\ \([2,090,103 \pm 11,607]\) \\ \(\langle100,788 \pm 452\rangle\)}   & \thead{\bestCovr{50} \\ \((581 \pm 10.5)\) \\ \([58,509,834 \pm 285,036]\) \\ \(\langle1,717,974 \pm 10,015\rangle\)}   & \thead{\bestCovr{50} \\ \((972 \pm 16.0)\) \\ \([90,024,795 \pm 604,542]\) \\ \(\langle3,208,338 \pm 17,172\rangle\)}      &                                                                                                               \\ \hline
\ftvi     & \hLMcut   & \thead{\bestCovr{50} \\ \((20.7 \pm 0.257)\) \\ \([991,116 \pm 14,912]\) \\ \(\langle460,168 \pm 6,570\rangle\)}   & \thead{\bestCovr{50} \\ \((746 \pm 13.7)\) \\ \([21,539,497 \pm 356,024]\) \\ \(\langle7,726,282 \pm 108,583\rangle\)}  & \thead{37 \\ \((1,621 \pm 29.0)\) \\ \([36,677,426 \pm 771,289]\) \\ \(\langle14,835,983 \pm 193,833\rangle\)}  &                                                                                                               \\ \hline
\ilao     & \hLMcut   & \thead{\bestCovr{50} \\ \((22.7 \pm 0.212)\) \\ \([4,444,825 \pm 21,791]\) \\ \(\langle111,278 \pm 490\rangle\)}   & \thead{\bestCovr{50} \\ \((1,082 \pm 12.2)\) \\ \([129,029,339 \pm 518,976]\) \\ \(\langle1,893,490 \pm 8,924\rangle\)} & \thead{40 \\ \((1,735 \pm 12.3)\) \\ \([195,389,276 \pm 910,083]\) \\ \(\langle3,527,186 \pm 17,438\rangle\)}  &                                                                                                               \\ \hline
\lrtdp    & \hLMcut   & \thead{\bestCovr{50} \\ \((17.3 \pm 0.114)\) \\ \([3,018,205 \pm 12,805]\) \\ \(\langle321,326 \pm 5,765\rangle\)} & \thead{\bestCovr{50} \\ \((601 \pm 15.2)\) \\ \([67,063,225 \pm 207,619]\) \\ \(\langle4,924,764 \pm 104,716\rangle\)}  & \thead{49 \\ \((1,490 \pm 33.9)\) \\ \([115,518,704 \pm 447,329]\) \\ \(\langle9,968,737 \pm 218,595\rangle\)}  &                                                                                                               \\ \hline
\cgilao   & \hMax     & \thead{\bestCovr{50} \\ \((10.6 \pm 0.307)\) \\ \([2,090,103 \pm 11,607]\) \\ \(\langle100,788 \pm 452\rangle\)}   & \thead{\bestCovr{50} \\ \((473 \pm 7.72)\) \\ \([58,509,834 \pm 285,036]\) \\ \(\langle1,717,974 \pm 10,015\rangle\)}   & \thead{\bestCovr{50} \\ \((737 \pm 13.1)\) \\ \([90,024,795 \pm 604,542]\) \\ \(\langle3,208,338 \pm 17,172\rangle\)}      & \thead{11 \\ \((1,746 \pm 19.6)\) \\ \([201,475,187 \pm 2,092,087]\) \\ \(\langle6,584,678 \pm 88,077\rangle\)} \\ \hline
\ftvi     & \hMax     & \thead{\bestCovr{50} \\ \((11.7 \pm 0.189)\) \\ \([991,116 \pm 14,912]\) \\ \(\langle460,168 \pm 6,570\rangle\)}   & \thead{\bestCovr{50} \\ \((520 \pm 11.5)\) \\ \([21,539,497 \pm 356,024]\) \\ \(\langle7,726,282 \pm 108,583\rangle\)}  & \thead{\bestCovr{50} \\ \((1,147 \pm 38.7)\) \\ \([38,153,835 \pm 990,907]\) \\ \(\langle15,053,868 \pm 212,968\rangle\)}  &                                                                                                               \\ \hline
\ilao     & \hMax     & \thead{\bestCovr{50} \\ \((18.7 \pm 0.249)\) \\ \([4,444,825 \pm 21,791]\) \\ \(\langle111,278 \pm 490\rangle\)}   & \thead{\bestCovr{50} \\ \((992 \pm 11.2)\) \\ \([129,029,339 \pm 518,976]\) \\ \(\langle1,893,490 \pm 8,924\rangle\)}   & \thead{\bestCovr{50} \\ \((1,604 \pm 16.8)\) \\ \([195,835,939 \pm 828,889]\) \\ \(\langle3,536,389 \pm 16,793\rangle\)}   &                                                                                                               \\ \hline
\lrtdp    & \hMax     & \thead{\bestCovr{50} \\ \(\bestMean{(9.80 \pm 0.083)}\) \\ \([3,013,611 \pm 10,921]\) \\ \(\langle321,347 \pm 5,796\rangle\)} & \thead{\bestCovr{50} \\ \((415 \pm 8.91)\) \\ \([66,996,690 \pm 214,188]\) \\ \(\langle4,923,361 \pm 104,809\rangle\)}  & \thead{\bestCovr{50} \\ \((1,030 \pm 31.3)\) \\ \([115,558,245 \pm 461,900]\) \\ \(\langle10,022,348 \pm 213,992\rangle\)} &                                                                                                               \\ \hline
\end{tabular}%
}
\covTabCaption{\tireworld}
\label{tab:tireworld}
\end{table*}

\end{document}